\documentclass[11pt]{article}

\usepackage{fullpage}
\usepackage[round]{natbib}
\usepackage{algorithm}
\usepackage[noend]{algorithmic}
\usepackage{amsmath,amsthm,amsfonts,amssymb}
\usepackage{amsmath}
\usepackage{hyperref}
\usepackage{color}
\usepackage{mathrsfs}
\usepackage{enumitem}
\usepackage{bm}
\usepackage{multirow}
\usepackage{booktabs}
\usepackage{makecell}
\usepackage{graphicx}
\usepackage{subfigure}
\usepackage{caption}
\usepackage{thmtools}
\usepackage{thm-restate}
\usepackage{setspace}
\usepackage{makecell}
\definecolor{light-gray}{gray}{0.85}
\usepackage{colortbl}
\usepackage{xcolor}
\usepackage{hhline}
\usepackage{xspace} 
\usepackage{mathtools} 
\usepackage{cancel}

\include{packages}
\usepackage{amsfonts}

\newcommand{\defeq}{\mathrel{\mathop:}=}
\newcommand{\vect}[1]{\ensuremath{\mathbf{#1}}}
\newcommand{\mat}[1]{\ensuremath{\mathbf{#1}}}

\newcommand{\argmax}{\mathop{\rm argmax}}

\newcommand{\trans}{^{\top}}
\newcommand{\diag}{\mathrm{diag}}
\newcommand{\poly}{\mathrm{poly}}

\newcommand{\norm}[1]{\|{#1} \|}

\newcommand{\E}{\mathbb{E}}

\renewcommand{\P}{\mathbb{P}}

\newcommand{\cO}{\mathcal{O}}
\newcommand{\tlO}{\mathcal{\tilde{O}}}

\newcommand{\N}{\mathbb{N}}
\newcommand{\R}{\mathbb{R}}

\newcommand{\A}{\mat{A}}
\newcommand{\B}{\mat{B}}

\newcommand{\U}{\mat{U}}

\newcommand{\M}{\mat{M}}
\newcommand{\X}{\mat{X}}

\newcommand{\e}{\vect{e}}
\renewcommand{\b}{\vect{b}}
\renewcommand{\u}{\vect{u}}

\renewcommand{\a}{\vect{a}}
\renewcommand{\o}{\vect{o}}

\newcommand{\fS}{\mathfrak{S}}

\newcommand{\cF}{\mathcal{F}}

\newcommand{\cD}{\mathcal{D}}

\newcommand{\cB}{\mathcal{B}}
\newcommand{\cX}{\mathcal{X}}

\newenvironment{proof-sketch}{\noindent{\bf Proof Sketch}
  \hspace*{1em}}{\qed\bigskip\\}
\newenvironment{proof-idea}{\noindent{\bf Proof Idea}
  \hspace*{1em}}{\qed\bigskip\\}
\newenvironment{proof-of-lemma}[1][{}]{\noindent{\bf Proof of Lemma {#1}}
  \hspace*{1em}}{\qed\bigskip\\}
\newenvironment{proof-of-proposition}[1][{}]{\noindent{\bf
    Proof of Proposition {#1}}
  \hspace*{1em}}{\qed\bigskip\\}
\newenvironment{proof-of-theorem}[1][{}]{\noindent{\bf Proof of Theorem {#1}}
  \hspace*{1em}}{\qed\bigskip\\}
\newenvironment{inner-proof}{\noindent{\bf Proof}\hspace{1em}}{
  $\bigtriangledown$\medskip\\}
\newenvironment{proof-attempt}{\noindent{\bf Proof Attempt}
  \hspace*{1em}}{\qed\bigskip\\}

\newcommand{\Fcal}{\mathcal{F}}

\newcommand{\Ocal}{\mathcal{O}}

\newcommand{\omle}{\textsf{OMLE}\xspace}

 \newtheorem{theorem}{Theorem}
 \newtheorem{lemma}[theorem]{Lemma}
 \newtheorem{corollary}[theorem]{Corollary}
 \newtheorem{remark}[theorem]{Remark}
\newtheorem{claim}{Claim}
 \newtheorem{proposition}[theorem]{Proposition}
 
 \theoremstyle{definition}
 \newtheorem{definition}[theorem]{Definition}
 
\newtheorem{assumption}{Assumption}

\renewcommand{\epsilon}{\varepsilon} 

\renewcommand{\O}{\mathbb{O}}
\renewcommand{\M}{\mathbb{M}}
\newcommand{\T}{\mathbb{T}}

\newcommand\numberthis{\addtocounter{equation}{1}\tag{\theequation}}

 \hypersetup{
     colorlinks,
     linkcolor={blue!50!black},
     citecolor={blue!50!black},
 }
 \colorlet{linkequation}{blue}

\begin{document}

 \title{\fontsize{15pt}{15pt}\textbf{When Is Partially Observable Reinforcement Learning Not Scary?}}

 \author{
 Qinghua Liu\footnotemark[1] 
 \and
 Alan Chung\footnotemark[1] 
 \and
 Csaba Szepesv{\'a}ri\footnotemark[2] 
 \and 
 Chi Jin\footnote{The author emails are \texttt{\{qinghual, alan.chung, chij\}@princeton.edu} and \texttt{szepesva@ualberta.ca}
 }
}

\date{\footnotemark[1] Princeton University, ~ \footnotemark[2] DeepMind and University of Alberta.}

 \maketitle
 
 \begin{abstract}%
Applications of  Reinforcement Learning (RL), in which agents learn to make a sequence of decisions despite lacking complete information about the latent states of the controlled system, that is, they act under partial observability of the states, are ubiquitous.
Partially observable RL can be notoriously difficult---well-known information-theoretic 
results show that learning partially observable Markov decision processes (POMDPs) requires an exponential number of samples in the worst case. 
Yet, this does not rule out the existence of large subclasses of POMDPs over which learning is tractable.

In this paper we identify such a subclass, which we call \emph{weakly revealing} POMDPs. This family rules out the pathological instances of POMDPs where observations are uninformative to a degree that makes learning hard.
We prove that for weakly revealing POMDPs,
a simple algorithm combining \emph{optimism} and \emph{Maximum Likelihood Estimation} (MLE)
is sufficient to guarantee polynomial sample complexity.
To the best of our knowledge, this is the first provably sample-efficient result for learning from interactions
in \emph{overcomplete POMDPs}, where the number of latent states can be larger than the number of observations.
 \end{abstract}


\section{Introduction}

A wide range of modern artificial intelligence challenges can be cast as Reinforcement Learning (RL) problems under \emph{partial observability}, in which agents learn to make a sequence of decisions despite lacking  complete information about the underlying state of system. For example, in robotics the agent has to cope with noisy sensors, occlusions, and unknown dynamics \citep{akkaya2019solving}, while in imperfect information games the player makes only local observations \citep{vinyals2019grandmaster,brown2019superhuman}. Further applications of partially observable RL include autonomous driving \citep{levinson2011towards}, resource allocation \citep{bower2005resource}, medical diagnostic systems \citep{hauskrecht2000planning}, recommendation \citep{li2010contextual}, business management \citep{de2009inventory}, etc. As such, learning and acting under partial observability has been an important topic in operation research, control, and machine learning.

Because of the non-Markovian nature of the observations,
learning and planning in partially observable environments requires an agent to maintain \emph{memory} and possibly reason about \emph{beliefs} over the states,
all while exploring to collect information about the environment. 
As such, partial observability can significantly complicate learning and planning under uncertainty.
While practical RL systems have succeeded in a set of partially observable problems including Poker \citep{brown2019superhuman}, Starcraft \citep{vinyals2019grandmaster} and certain robotic tasks \citep{cassandra1996acting}, the theoretical understanding of learning to act in partially observable systems remains limited.  Most existing results in RL theory focus on  fully observable systems or, more generally, learning when the features of states are accessible and can faithfully represent value functions.
As such, algorithms developed for this case need not to reason about what the latent state may be and in particular do not need to resort to using the observation histories.
Thus, the resulting algorithms can be fundamentally limited and may not work beyond the narrow settings that they are designed for.
Owning to the ubiquity of partially observable problems, addressing the theoretical challenges of partial observability is vital to closing the gap between the typical applications and the scope of available theoretical works.

This paper considers Partially Observable Markov Decision Process (POMDPs)---the standard model in reinforcement learning that captures the partial-information structure. Despite the existence of many efficient algorithms for learning MDPs in the fully observable settings, learning POMDPs is notoriously difficult in theory---well-known complexity-theoretic results show that learning and planning in partially observable environments is indeed statistically 
and computationally \emph{intractable} in general~\citep{papadimitriou1987complexity,mundhenk2000complexity,vlassis2012computational,mossel2005learning}, even if in the favorable setting with a small number of states, actions, and observations. However, these complexity barriers are of a worst case nature, and they do not preclude efficient algorithms for learning rich sub-classes of POMDPs which could potentially cover interesting practical applications. This leaves an important question: 
\begin{center}
\textbf{Can we identify a rich sub-class of POMDPs that empowers sample-efficient RL?}
\end{center}

Prior efforts on sample-efficient learning of POMDPs focus either on special cases of POMDPs such as latent MDPs \citep{kwon2021rl,kwon2021reinforcement}, or on general POMDPs but with restrictive assumptions. In particular, \citet{azizzadenesheli2016reinforcement,guo2016pac} do not address strategic exploration---a core challenge in RL; \citet{jin2020sample} considers the exploration setting but only addresses undercomplete POMDPs, where the number of states must be no larger than the number of observations. 

This paper answers the highlighted question above affirmatively. We identify a rich family of tractable POMDPs---\emph{weakly revealing} POMDPs (see Section \ref{sec:weakly}), which rule out the pathological instances whose observations contain no information to distinguish latent states. \emph{Weakly revealing} POMDPs are very rich---it contains a majority of existing POMDPs classes which are known to be tractable \citep{azizzadenesheli2016reinforcement,guo2016pac,jin2020sample}; it also handles overcomplete POMDPs where the number of latent states can be larger than the number of observations. 

We further propose a new simple algorithm for learning POMDPs---\emph{Optimistic Maximum Likelihood Estimation} (\omle). As its name suggests, \omle (read, Oh-Em-El-Eeh) 
combines optimism with classical maximum likelihood estimation \citep{wilks1938large}. In contrast to the algorithm of \citet{jin2020sample} which heavily exploit the undercomplete structure, our algorithm is generic, does not explicitly rely on any special structure, and can be used for any POMDPs. We prove that \omle learns a near-optimal policy for any weakly revealing POMDP within a polynomial number of samples (Theorem \ref{thm:under} and \ref{thm:over}). To the best of our knowledge, this is the first provably sample-efficient result for learning overcomplete POMDPs in settings where exploration is necessary. Our result also reasserts that optimism is a powerful tool to address exploration needs, regardless of whether states are observable.
We complement our positive results with lower bounds showing that certain polynomial dependency on the problem parameters in our sample complexity is necessary.

Finally, we remark that 
our algorithm, as well as all existing algorithms for learning large classes of POMDPs, remains \emph{computationally} inefficient. This is due to the inherent computational hardness of learning POMDPs: planning (i.e., computing the optimal policy \emph{given} model parameters) alone is already PSPACE-complete \citep{papadimitriou1987complexity}, not mentioning the additional computation required for model estimation and exploration. We leave the challenge of computationally efficient learning for future work.

\subsection{Overview of techniques}
The major technical challenge of this paper is to establish the sample efficiency guarantee of \omle for learning any weakly revealing POMDPs despite the simplicity of the algorithm. Our results rely on the following three key ideas. To the best of our knowledge, the second and the third ideas are novel in the context of learning POMDPs, while the first technique was used by \citet{jin2020sample}.

\begin{itemize}
\item \textbf{Observable Operator Model (OOM) \citep{jaeger2000observable}}: 
OOM provides an alternative parameterization of the POMDP model, by representing the probability of a trajectory over observations and actions as the product of a series of linear operators, which is known as \emph{observable operators}. For more details, see Section \ref{subsec:oom}. Such linear structure facilitates us to use existing tools from matrix analysis to analyze POMDPs.
Although \omle algorithm does not explicitly utilize the OOM representation, the observable operators serve as important intermediate quantities in our analysis. They help us to bound the suboptimality of the learned policy as a function of the size of our confidence set.

\item \textbf{MLE-based Confidence Set:} In contrast to the current mainstream approaches of learning POMDPs which use \emph{spectral methods} to directly estimate either the model parameters or the observable operators \citep[see, e.g.,][]{ azizzadenesheli2016reinforcement,guo2016pac,jin2020sample}, we use the \emph{maximum likelihood estimation} (MLE) approach, which provides implicit guarantees on learning observable operators. We achieve this by adapting the classic techniques for analyzing MLE \citep[e.g.,][]{geer2000empirical}. An appealing feature of the MLE approach is its generality and that the confidence set construction does not need to rely on the specific structure of the problem. 
The strength of this unified approach is that it allows \omle to be used with almost no changes in both undercomplete and overcomplete POMDPs. In constrast, spectral-based algorithms require more careful designs that are adjusted to specific problems (such as undercomplete vs. overcomplete settings). These adjustments, if not done optimally, easily lead to requirement of unnecessary, artificial assumptions. 

\item \textbf{$\ell_1$-norm eluder Dimension:} To prove the sample efficiency of optimistic algorithms, one needs to argue that, after a sufficient number of iterations, the size of the maintained confidence set is small enough to guarantee near-optimality of the learned policy. 
In the tabular setting,
this is typically achieved by resorting to the pigeon-hole principle \citep[e.g.,][]{azar2017minimax,jin2018q}, 
while in the linear setting, one typically uses the so-called elliptical potential lemma \citep[e.g.,][]{lattimore2020bandit}. To generalize these argument, 
\citet{russo2013eluder} introduced the notion of eluder dimension for sets of real-valued functions with a common domain.
The use of MLE-based confidence set requires us to develop a new result which is stronger than the standard elliptical potential arguments: While we have linear structures, the $\ell_2$-norms typically used are not suitable for our purposes. As such, the standard eluder dimension (which is tied to the $\ell_2$-norm) is also unsuitable. To address these challenges, we introduce a variation of the eluder dimension, which is called $\ell_1$-norm eluder dimension, and which might be of independent interest. 
\end{itemize}


\subsection{Related works} \label{sec:related_work}

Reinforcement learning has been extensively studied in the fully observable setting \citep[see, e.g.,][and the references therein]{azar2017minimax,dann2017unifying,jin2018q,jin2020provably, zanette2020learning,jiang2017contextual}. For the purpose of this paper, we focus our attention on reviewing the theoretical results for partially observable reinforcement learning.

\paragraph{Hardness of learning POMDPs.} There is a line of well-known  computational  hardness results for planning and learning in POMDPs. 
Firstly, even when the parameters of a POMDP are known, computing the optimal policy (i.e., planning) is  PSPACE-complete \citep{papadimitriou1987complexity}. Moreover, even if one only wants to find the optimal memoryless policy, the problem is still NP-hard \citep{vlassis2012computational}. In addition, 
the model estimation of POMDPs is also computationally hard---\citet{mossel2005learning} proved an average-case computational result showing that estimating the model parameters for a subclass of Hidden Markov Models (HMMs) is at least as hard as learning parity with noise\footnote{Learning parity with noise is conjectured to be NP-hard in the theory of computational complexity.}.
Since HMMs can be viewed as special cases of POMDPs without action control, their result directly implies estimating the model parameters of POMDPs is hard.  

Learning POMDPs is also known to be statistically hard: \citet{krishnamurthy2016pac} proved that finding a near-optimal policy of a POMDP in the worst case requires a number of samples that is exponential in the episode length. The hard instances are those pathological POMDPs where the observations contain no useful information for identifying the system dynamics.

\paragraph{Positive results for learning POMDPs.} 
Despite the worst-case hardness results, there is a long history of  learning sub-classes of POMDPs. 
\citet{even2005reinforcement} studied POMDPs without resets, where the proposed algorithm has sample complexity scaling exponentially with a certain horizon time. \citet{poupart2008model,ross2007bayes} developed  Bayesian methods to learn POMDPs, while \citet{azizzadenesheli2018policy} considered learning the optimal memoryless policies  with policy gradient methods. PAC or regret bounds are not known for these approaches. 

In the category of polynomial sample results, 
a sequence of recent works \citep[e.g.,][]{guo2016pac,azizzadenesheli2016reinforcement,jin2020sample,xiong2021sublinear} applied spectral methods \citep{hsu2012spectral,anandkumar2014tensor} to learning POMDPs and obtained polynomial sample complexity results. Among them, \citet{guo2016pac,azizzadenesheli2016reinforcement,xiong2021sublinear} 
made strong reachability assumptions and did not address the exploration problem.
Furthermore, these results assume that both the transition and emission matrices are full rank, which are stronger than the weakly revealing conditions considered in this paper. 
\citet{jafarnia2021online} proposed a posterior sampling-based algorithm, and provided sample-efficient guarantees \emph{assuming} either sufficient separability between different models, or the success of belief state and transition kernel estimation. These assumptions significantly reduce the difficulty of estimating model dynamics---a core challenge in learning POMDPs, and thus reduce the generality of the results.

The most related work to us is \citet{jin2020sample}, which addressed the exploration problem in learning undercomplete POMDPs, where the number of latent states must be no greater than the number of observations. Their algorithm is specially designed to exploit the undercomplete structure of POMDPs. It remains unclear if their techniques can be extended to the overcomplete setting. In contrast, this paper presents a new generic algorithm based on MLE, which enjoys provable sample-efficiency in the exploration settings of both undercomplete and overcomplete POMDPs.

Very recently, \citet{golowich2022planning} developed the first quasi-polynomial time planning algorithm for a subclass of POMDPs. Their result holds under the $\gamma$-observability condition, which is very similar to the weakly-revealing condition presented in this paper.\footnote{$\gamma$-observability condition requires that $\min_{h}\|\O_h(b-b')\|_1\ge \gamma \|b-b'\|_1$ for any $b,b\in\Delta_{S}$, while $\alpha$-weakly-revealing condition (Assumption \ref{asp:under}) assumes $ \min_h\|\O_h x\|_2\ge \alpha \|x\|_2 $ for any $x \in \R^O$. 
In Lemma \ref{lem:equiv}, we prove that  $\frac{\alpha}{\sqrt{S}} \le \gamma \le 4\sqrt{O}\alpha$. As a result, these two conditions are ``equivalent'' up to  a factor of at most  $\cO(\sqrt{O})$.} Compared to this paper, the result in \citet{golowich2022planning}  purely focuses on the computational efficiency. It is restricted to the undercomplete setting, and addresses only planning but not estimation or exploration, all of which are important components for learning POMDPs.

\paragraph{Latent MDPs.}
Latent MDPs \citep{kwon2021rl}---where an MDP is randomly drawn from a set of $M$ possible MDPs at the beginning of the interaction---can be considered as a special class of overcomplete POMDPs. \citet{kwon2021rl} proved that learning latent MDPs remains statistically hard in the worst case. They also provided  several positive results for learning latent MDPs with additional assumptions, such as revealing the latent contexts at the end of each episode.
\citet{kwon2021reinforcement} provided positive results for latent MDPs without these additional assumptions, but the results only apply to the setting of $M=2$ with a shared transition. Latent MDPs and weakly revealing POMDPs do not contain each other.

\paragraph{Decodable POMDPs.} 
Block MDPs \citep{krishnamurthy2016pac} are POMDPs whose current latent state can be uniquely determined by the current observation. 
By simple algebra, one can verify that block MPDs are special cases of single-step weakly revealing POMDPs that satisfy Assumption \ref{asp:under} with $\alpha\ge 1/\sqrt{O}$.
The recently proposed $m$-step decodable POMDPs \citep{efroni2022provable} are  generalizations of block MDPs, in which the latent state can be uniquely decoded from the most recent history (of observations and actions) of a short length $m$. This multistep decodability assumption can also be viewed as a special case of general ``weakly revealing''-type of conditions. However, \citet{efroni2022provable} assume that $m$-step history decodes (weakly reveals) the current state, while this paper assumes that $m$-step future weakly reveals the current state.
Finally, we remark that most existing results for block MDPs or $m$-step decodable POMDPs \citep[see, e.g.,][]{krishnamurthy2016pac,jiang2017contextual,du2019provably, misra2020kinematic,efroni2022provable} further involve decoder class or value function approximation, which is beyond the scope of this paper.


\paragraph{RL with function approximation.} There is a recent line of research \citep[e.g.,][]{jiang2017contextual,ayoub2020model,du2021bilinear,jin2021bellman,foster2021statistical} on reinforcement learning with general function approximation.
This line of results proposed  certain complexity measure for sequential decision making problems, and developed generic algorithms which have sample-efficient guarantees as long as the complexity measure of RL problems is small. These frameworks are known to cover a special subclass of POMDPs---reactive POMDPs \citep{jiang2017contextual}, where the optimal value only depends on the current-step observation-action pair. 
It remains highly unclear whether weakly revealing POMDPs identified in this paper can be covered by those general frameworks. We remark that investigating this problem requires us to compute those complexity measures for POMDPs, which is highly non-trivial and may require techniques developed in this paper.

\paragraph{MLE approaches in bandit and RL.}
The idea of using the MLE principle in the confidence set construction can be traced back to \citet{lai1987adaptive}, which considers   the problem of Bernoulli bandits. MLE-based approaches are also used in the framework of reward-biased MLE \citep{kumar1982new,mete2021reward}, which balances the reward with the likelihood value, for learning tabular MDPs. Recently, the MLE-based approaches are also used in the setting of representation learning in RL \citep{agarwal2020flambe,uehara2021representation}.


\newcommand{\bigO}{\mathcal{O}}
\newcommand{\bigOmega}{\mathcal{O}}
\newcommand{\cT}{\mathcal{T}}
\newcommand{\pomdp}{\text{POMDP}}
\renewcommand{\fS}{\mathscr{S}}
\newcommand{\fA}{\mathscr{A}}
\newcommand{\fO}{\mathscr{O}}

\section{Preliminaries}
For a positive integer $n$, we let $[n] = \{1,\dots,n\}$.
We consider episodic, tabular, partially observable Markov decision processes (POMDP).
These processes generalize the standard Markov decision processes 
by making agents observe a ``noisy function'' of the state of a controlled Markov process. 
We consider time inhomogeneous, fixed horizon version of POMDPs. 
Formally, such a POMDP is specified by a tuple
$(\fS,\fA,\fO;H,\mu_1,\T,\O;r)$. Here $\fS,\fA$ and $\fO$ denote the space of state, action and  observation respectively, with respective cardinalities $|\fS|=S$, $|\fA|=A$ and $|\fO|=O$; $H$ denotes the length of each episode; $\mu_1\in\Delta_S$ denotes the  distribution of the initial state where $\Delta_S$ is the $(S-1)$-dimensional 
probability simplex which we identify with the set of distributions over the states $\fS$; 
$\T=\{\T_{h,a}\}_{(h,a)\in[H-1]\times\fA}$ denotes the collection of transition matrices where $\T_{h,a}$ is the $S\times S$ \emph{transition matrix} of action $a$ at step $h$ such that $\T_{h,a}(\cdot\mid s)$ gives the distribution of the next state if the agent takes action $a$ at state $s$ and step $h$; 
$\O=\{\O_h\}_{h\in[H]}$ denotes the collection of  \emph{emission matrices} of size $O\times S$ 
so that $\O_h(\cdot\mid s)$ gives the distribution over observations at step $h$ conditioned on the current hidden state being $s$; and and $r=\{r_h\}_{h\in[H]}$ are the known reward functions from $\fO$ to $[0,1]$  such that the agent will receive reward $r_h(o)$ when she observes $o\in\fO$ at step $h$. \footnote{This is equivalent to assuming that reward information is contained in the observation. We consider this setup to avoid the leakage of information about the latent states through rewards beyond observations. We remark that all  results in this paper immediately extend to the more general setting where reward $r(\tau_H)$ can be a function of the entire observation-action trajectory $\tau_H$, and is only received at the end of each episode.}

In a POMDP, the states are generally hidden from the agent: in every step a controlling agent can only see the observations and her  own actions. At the beginning of each episode, the environment samples an initial state $s_1$ from $\mu_1$. 
At each step $h\in[H]$, the agent first observes $o_h$ that is sampled from $\O_h(\cdot\mid s_h)$, the observation distribution of hidden state $s_h$ at step $h$. Then the agent receives reward $r_h(o_h)$ that is computed from $o_h$, and takes action $a_h$.
After this, the environment transitions to $s_{h+1}$, whose distribution follows $\T_{h,a_h}(\cdot\mid s_h)$. The current episode terminates immediately after $a_{H}$ is taken. We use $\tau_h = (o_1,a_1,\ldots,o_h,a_h)$ to denote a trajectory from step $1$ to step $h$.

A policy $\pi=\{\pi_h:~\cT_h\rightarrow \Delta_A\}_{h=1}^H$ is a collection of $H$ functions where $\cT_h=(\fO\times\fA)^{h-1}\times\fO$ denotes the set of all length-$h$ histories.  
Given a policy $\pi$, we use $V^\pi$ to denote its value, which is defined as the expected total reward received under policy $\pi$:  
\begin{equation}
    V^\pi := \E_\pi \left[ \sum_{h=1}^H r_h(o_h)\right],
\end{equation}
where the expectation is with respect to the randomness of the transitions, observations and the policy. 
Since the state, action, observation spaces and the horizon are all finite,
there always exists an optimal policy $\pi^\star$ that achieves the optimal value $V^\star:=\sup_\pi V^\pi$. Different from MDPs, the optimal policies in POMDPs are in general history-dependent instead of only depending on the current observation,
which makes not only learning, but already computing a near-optimal policy in known POMDPs more challenging than doing the same in MDPs.

\paragraph{Learning objective.} Our goal is to learn an $\epsilon$-optimal policy $\pi$ in the sense that $V^\pi \ge V^\star-\epsilon$, using a number of samples polynomial in all relevant parameters. We also consider the problem of learning with low regret. Suppose the agent interacts with POMDPs for $K$ episodes, and plays a policy $\pi_k$ in the $k^\text{th}$ iteration for any $k \in [K]$. The total (expected) regret is then defined as:
\begin{equation*}
 \text{Regret}(K) = \sum_{k=1}^K [V^\star - V^{\pi_k}] .
\end{equation*}
The question then is whether a learner can keep the regret small.

\paragraph{Notation.} 
We use bold upper-case letters  $\B$ to denote matrices and bold lower-case letters $\b$ to denote vectors.  
Given a matrix $\B\in\R^{m\times n}$, we use $\B_{ij}$ to denote its $(i,j)^{\rm th}$ entry, $\sigma_{k}(\B)$ to denote its $k^\text{th}$ largest singular value, and $\B^\dagger$ to denote its Moore-Penrose inverse.
For a vector $\b\in\R^m$, we use $\diag(\b)$ to denote a diagonal  matrix with $[\diag(\b)]_{ii}=\b_i$.


\section{Weakly Revealing POMDPs}
\label{sec:weakly}

The purpose of this section is to 
define the class of weakly revealing POMDPs. We first motivate our definition by revisiting the pathological instances which prevent sample-efficient learning of POMDPs  in general. We then introduce the formal definition of weakly revealing POMDPs in the \emph{undercomplete}
 setting when
$S \le O$ and finally extend it to the \emph{overcomplete} setting when $S > O$.
All the proofs for this section are deferred to  Appendix \ref{app:weakly}.

\subsection{Hard instances of POMDPs}
Here we revisit the hardness results and the pathological instances constructed by \citet{krishnamurthy2016pac} and \citet{jin2020sample}.
As it turns out, learning POMDPs is statistically hard in the worst-case due to the existence of POMDPs with uninformative observations. 

\begin{proposition}[\citet{krishnamurthy2016pac,jin2020sample}]
There exists a class of $2$-states  $H$-horizon POMDPs whose observations reveal no information about the underlying states up to the end, such that any algorithm requires at least $A^{\Omega(H)}$ samples to learn an $\cO(1)$-optimal policy with a probability of $1/2$ or higher.
\end{proposition}

The hard instance is a combinatorial lock with unobserved states.
Consider POMDPs with states $s_\text{h,good}$ and $s_\text{h,bad}$, $h=1,\dots,H$. 
The emission probability $\O_h(\cdot|s_\text{h,good})$ is precisely the same as $\O_h(\cdot|s_\text{h,bad})$ for the all steps except the last one, so that the agent has absolutely no information about the latent state during the first $H-1$ steps. 
Let the initial state be $s_{\text{1,good}}$.
Consider a special action sequence $\{a^\star_h\}_{h=1}^{H-1}$, and construct the transition dynamics such that at each step $1\le h\le H-1$, the next state is $s_\text{h+1,good}$ 
only if the previous state is  $s_\text{h,good}$ and the action taken is $a^\star_h$. In all other cases, the environment  transitions to $s_\text{h+1,bad}$. Finally, the agent will receive a reward of one only if she is in $s_\text{H,good}$ at step $H$; the agent receives zero reward otherwise.

It is not hard to see the optimal policy will  take action $a^\star_h$ at step $h$, which will give a total reward of $1$. However, since the agent effectively has no observation in the first $H-1$ steps, she has no option but to try out all possible action sequences, which requires $A^{\Omega(H)}$ episodes to find the correct action sequence with constant probability.

\subsection{Weakly revealing condition in the undercomplete setting} 

Based on the hard instances constructed above, we conclude that if the observations do not contain information to distinguish  two different latent states, then learning these POMDPs is statistically hard. For POMDPs with more than two states, the hardness result above can be easily extended to the case where there exist two mixtures of latent states with disjoint support such that the observations do not contain any information to distinguish these two mixtures. Concretely, by a mild abuse of language, a mixture of states is identified by a probability vector $\nu \in \Delta_S$; $\nu_1$ and $\nu_2$ are said to have disjoint support if $\text{supp}(\nu_1) \cap \text{supp}(\nu_2)=\emptyset$.

A direct approach to rule out the above-described pathological instances is 
to just assume that any two
latent state
mixtures  $\nu_1, \nu_2$ that have disjoint support induce distinct distributions over observations, that is, $\O_h\nu_1\neq\O_h\nu_2$ for all $h\in[H]$ 
where $\O_h$ is the $O\times S$ emission matrix at step $h$.
A linear algebraic argument then shows 
that this condition is equivalent to that the rank of the emission matrix $\O_h$ is $S$. 

\begin{proposition}\label{prop:property_singlestep_weak_revealing}
The emission matrix $\O_h$ is rank $S$ if and only if the induced distributions over observations are distinct for any two mixtures of latent states with disjoint support.
\end{proposition}

The weakly revealing condition is simply a robust version of the condition that the rank of the emission matrices is $S$---it assumes the $S^\text{th}$ singular value of emission matrix $\O_h$ is lower bounded. This condition in the undercomplete setting was first identified by \cite{jin2020sample} as a technical condition to ensure the sample efficiency of their algorithms.
\begin{assumption}[$\alpha$-weakly revealing condition]\label{asp:under}
There exists $\alpha >0$, such that $\min_h \sigma_{S}(\O_h)\ge\alpha$.
\end{assumption}
This  condition ensures that the observations contain enough information to distinguish any two mixtures of states given a sufficiently large number of samples.

We call Assumption \ref{asp:under} the ``\emph{weakly}'' revealing condition to distinguish it from the setup known as rich observation or block MDP in the literature \citep{jiang2017contextual,du2019provably,misra2020kinematic}. The latter setup considers the problem where the latent state can be directly recovered from any single observation and the stage $h$ in the episode. That is, the latent state is completely revealed by the observation. Therefore, technically speaking, block MDPs are fully observable, which is in a way  ``diagonally opposite'' to the setting we consider.

Finally, we note that since $\O_h$ is a matrix of size $O \times S$, Assumption \ref{asp:under} implicitly requires $S \le O$. That is, it only holds in the undercomplete setting.

\subsection{Weakly revealing condition in the overcomplete setting}
In the overcomplete setting, we have $S > O$. It is information-theoretically impossible to distinguish any two mixtures of latent states by inspecting observations only in a single step. 
The key observation
here is that we should instead inspect the distribution of observations for \emph{$m$ consecutive steps}. We note that the number of all possible observable sequence $(o_1, a_1,\ldots, a_{m-1}, o_m)$ of length $m$ is $O^{m}A^{m-1}$, which is larger than $S$ when $m \ge \Omega(\log S)$.

To state our assumption, we 
define the $m$-step emission-action matrices 
\[
\{\M_h\in\R^{(A^{m-1}O^m)\times S}\}_{h\in[H-m+1]}
\]
as follows: 
For an observation sequence $\o$ of length $m$, initial state $s$ and action sequence $\a$ of length $m-1$,
we let $[\M_h]_{(\a,\o),s}$ be the probability of receiving  $\o$ provided that the action sequence $\a$ is used from state $s$ and step $h$: 
 \begin{equation}\label{defn:M}
     [\M_h]_{(\a,\o),s}= \P(o_{h:h+m-1}=\o \mid 
     s_h=s,a_{h:h+m-2} =\a )\quad \text{for all } (\a,\o)\in \fA^{m-1}\times\fO^m  \text{ and } s\in\fS.
 \end{equation}
Similar to the undercomplete case, the weakly revealing condition in the overcomplete setting assumes 
that the $S^\text{th}$ singular value of the $m$-step emission matrix $\M_h$ is lower bounded.
\begin{assumption}[$m$-step $\alpha$-weakly revealing condition]\label{asp:over}
There exists $m \in \N$, $\alpha > 0$ such that $\min_{h\in[H-m+1]} \sigma_{S}(\M_h)\ge\alpha$ where $\M_h$ is the $m$-step emission matrix defined in \eqref{defn:M}.
\end{assumption}
Assumption \ref{asp:over} ensures that the observable sequence in the next $m$ consecutive steps contain enough information to distinguish any two mixtures of states given a sufficiently large number of observations. 
Assumption \ref{asp:under} is a special case of Assumption \ref{asp:over} with $m=1$.

Finally, we remark that in case that $O^m \ge S$, a sufficient condition to make Assumption \ref{asp:over} hold is that: for any stage $h$, there exists a $(m-1)$-step action sequence such that the $m$-step observation sequences under this action sequence is $\alpha$-weakly revealing the hidden state. Formally, for any $h\in [H]$ and $\a\in \fA^{m-1}$ 
 let $\M_{h,\a}$ stands for the $O^m\times S$ matrix obtained from $\M_h$ by selecting the rows of $\M_h$ where the row-index corresponds to $\a$. That is, $(\M_{h,\a})_{\o,s} = [\M_h]_{(\a,\o),s}$.
\begin{proposition} \label{prop:property_multistep_weak_revealing}
Assume that $O^m\ge S$, then Assumption \ref{asp:over} holds if $\max_{a\in \fA^{m-1}} \sigma_S(\M_{h,\a}) \ge \alpha$ for all $h \in [H-m+1]$.
\end{proposition}


\section{Main Results}
\label{sec:main}
In this section, we present our algorithm---\emph{Optimistic Maximum Likelihood Estimation} (\omle) and its theoretical guarantees for learning weakly revealing POMDPs in both the undercomplete and the overcomplete settings.

\subsection{Undercomplete setting}
For clarity, we first present the algorithm and results for learning undercomplete POMDPs under Assumption~\ref{asp:under}. As we will see in the later section, 
with a minor modification this algorithm also
generalizes to learning overcomplete POMDPs under Assumption~\ref{asp:over}.

\paragraph{Algorithm description}

To condense notations, we use $\theta=(\T,\O,\mu_1)$ to denote the model parameters of a POMDP and use $\Theta$ to denote the collections of all possible model parameters $\theta$ that correspond to POMDPs with $S$ states, $A$ actions, and $O$ observations. 
 To make the dependence on $\theta$ explicit, we will use 
 $V^\pi(\theta)$ to denote the value of a policy $\pi$,
 while we use $\P^\pi_\theta(\tau)$ to denote the probability of observing a trajectory $\tau$ under policy $\pi$, when the underlying POMDP is given by $\theta$. We also use $\O_h(\theta)$ ($\M_h(\theta)$) to denote the emission matrix of  $\theta$ (respectively, the multistep emission matrix of  $\theta$).

Algorithm \ref{alg:under} gives the pseudocode of
\omle. 
As can be seen from this pseudocode,
in each episode $k$  there are two main steps: 
\begin{itemize}
  \item Optimistic planning (Lines \ref{line:under-1}-\ref{line:under-2}): 
  find the POMDP model $\theta^k$ with the highest optimal value in the confidence set $\cB^k$ and follow the associated optimal policy $\pi^k$ in the episode to collect a trajectory $\tau^k$. 
  \footnote{Our algorithm, as well as all existing algorithms for learning large classes of POMDPs, is computationally inefficient. In particular, a naive implementation of optimistic planning (Line \ref{line:under-1}) is to enumerate all POMDP models in an $\epsilon$-cover of the confidence set and compute their optimal policies, which requires $e^{\Omega(HS^2A+HSO)}$
  time in the worst case.}
  \item Confidence set update (Line \ref{line:under-3}): add the newly collected policy-trajectory pair into the dataset,  and then update the confidence set to include those models that assign a total log-likelihood 
  to the data
  that is ``close'' to the maximum possible such total log-likelihood.
  In particular, the form of the confidence set is
\begin{equation*}
 \bigg\{\hat\theta \in \Theta: \sum_{(\pi,\tau)\in\cD} \log \P_{{\hat\theta}}^{\pi} (\tau)
   \ge \max_{ \theta' \in\Theta} \sum_{(\pi,\tau)\in\cD} \log \P^{\pi}_{{\theta'}}(\tau) -\beta  \bigg\} \bigcap \cB^1,
\end{equation*}
	where $\cB^1$ is the initial confidence set that contains all $\alpha$-weakly revealing models of a given size. 
\end{itemize}
Compared to the standard maximum likelihood estimation (MLE) approach,
all $\alpha$-weakly revealing models with a sufficiently high likelihood are allowed and the size of this set is controlled by $\beta\ge 0$.
In particular, if $\beta=0$, the confidence set collapses to the solutions of MLE.

In our algorithm, the choice of $\beta$ is governed by the magnitude of the ``statistical noise'' introduced by various random events. By analyzing this noise, one can choose the value of $\beta$ to guarantee that the true POMDP model is always contained in the resulting confidence set with a prescribed probability (see Proposition \ref{prop:mle-optimism} for a rigorous statement). 

We emphasize that the algorithm design of MLE is considerably simpler than that of prior provably sample-efficient algorithms for learning POMDPs \citep[see, e.g.,][]{azizzadenesheli2016reinforcement,guo2016pac,jin2020sample}, which 
rely on spectral methods.

\begin{algorithm}[t]
   \caption{\textsc{Optimistic Maximum Likelihood Estimation (OMLE)}}
\begin{algorithmic}[1]\label{alg:under}
\STATE \textbf{Initialize:} $\cB^1 = \{ \hat \theta\in\Theta:~\min_h\sigma_{S}(\O_h(\hat\theta))\ge \alpha \}$, $\cD=\{\}$ 
   \FOR{$k=1,\ldots,K$}
   \STATE compute $(\theta^k,\pi^k) = \argmax_{\hat\theta\in\cB^k, \pi} V^\pi(\hat\theta)$ \label{line:under-1}
   \STATE execute policy $\pi^k$ to collect a trajectory $\tau^k :=(o^k_1,a^k_1,\ldots,o^k_h,a^k_h)$\label{line:under-2}
   \STATE add $(\pi^k,\tau^k)$ into $\cD$ and update
   \vspace{-3mm}
   \begin{equation}
   \cB^{k+1} = \bigg\{\hat\theta \in \Theta: \sum_{(\pi,\tau)\in\cD} \log \P_{{\hat\theta}}^{\pi} (\tau)
   \ge \max_{ \theta' \in\Theta} \sum_{(\pi,\tau)\in\cD} \log \P^{\pi}_{{\theta'}}(\tau) -\beta  \bigg\}\bigcap \cB^1
   \vspace{-3mm}
   \end{equation}\label{line:under-3}
   \ENDFOR
\end{algorithmic}
\end{algorithm}

\paragraph{Theoretical guarantees}\label{subsec:under-thm}
Our main result, which shows that \omle will achieve small regret in any \emph{weakly revealing POMDPs} (Assumption \ref{asp:under}), is as follows:

\begin{theorem}[Regret of \omle] \label{thm:under}
There exists an absolute constant $c>0$ such that for any $\delta\in(0,1]$ and $S,A,O,H,K\in\N$,  if we choose $\beta = c\left(H(S^2A+SO)\log(SAOHK)+\log(K/\delta)\right)$ in Algorithm \ref{alg:under}, then,
for any POMDP with $S$ states, $A$ actions, $O$ observations and horizon $H$ and
satisfying Assumption \ref{asp:under},
  with probability at least $1-\delta$,%
  $$   {\rm Regret}(k)\le \poly(S,A,O,H,\alpha^{-1},\log(\delta^{-1}K))\cdot\sqrt{k}  \qquad \text{ for all }k\in[K].$$
\end{theorem}
The proof, as well as the specific polynomial dependency, is presented in Appendix~\ref{app:proof-under}.
Note that the growth rate of regret as a function $k$ is optimal \citep{auer1995gambling}.

Moreover, by the standard online-to-batch conversion \citep{cesa2004generalization}, 
the regret bound immediately implies the following sample complexity result:
\begin{corollary}[Sample Complexity of \omle]\label{cor:sample-under}
Under the same setting as Theorem \ref{thm:under}, when
 $K \ge \poly(S,A,O,H,\alpha^{-1},\log(\epsilon^{-1}\delta^{-1}))\cdot\epsilon^{-2}$, with probability at least $1-\delta$, the uniform mixture of the policies produced by \omle is $\epsilon$-optimal. I.e., $(1/K)\cdot\sum_{k=1}^K V^{\pi^k} \ge V^\star - \epsilon$.
\end{corollary}
Here, the $\tilde{O}(\epsilon^{-2})$ dependence is also optimal up to log factors. 
Previous work by \citet{jin2020sample} also provides polynomial sample-complexity guarantee for learning \emph{$\alpha$-weakly revealing} POMDPs under Assumption \ref{asp:under}. The present result improves over the results of \citet{jin2020sample} in the following aspects: 
\begin{itemize}
  \item 
While the \textsf{OOM-UCB} algorithm of \citet{jin2020sample} heavily exploited the special structure of undercomplete POMDPs,
   \omle~appears in a much simpler form  and the algorithm design arguably does not use this special structure.
As a result, \omle~can be easily extended to learning \emph{multi-step weakly revealing} POMDPs, while to the best of our knowledge \textsf{OOM-UCB} cannot. 
  \item In terms of theoretical guarantees, \omle~enjoys a near-optimal $\sqrt{k}$-regret while \textsf{OOM-UCB} was only shown to achieve a regret of size $O(k^{2/3})$. The higher regret of \textsf{OOM-UCB} is  due to the limitation of its exploration mechanism.\footnote{\textsf{OOM-UCB} itself is not a no-regret algorithm. However, combining its $\tilde\cO(\poly(\cdot)/\epsilon^2)$ sample complexity guarantee with the explore-then-commit strategy implies a $\tilde\cO(\poly(\cdot)\times k^{2/3})$-regret.
  }
\end{itemize}

Finally, observe that the upper bound in Theorem \ref{thm:under} depends polynomially on the inverse of $\alpha$---an upper bound on the $\ell_1$-norm of the pseudoinverse of the  emission matrices in Assumption \ref{asp:under}. This polynomial dependence turns out to be unavoidable as is shown by the following lower bound.
\begin{theorem}[Necessity of $\poly(\alpha^{-1})$ dependency] \label{thm:lowerbound1}
  For any $\alpha \in \left( 0, 1/2 \right)$ and $H,A\in\N^+$, there exists an undercomplete $\alpha$-weakly revealing POMDP  with $S,O=\cO(1)$ so that any algorithm requires at least $\Omega(\min\{\frac{1}{\alpha H},A^{H-1}\})$ samples to learn a $(1/2)$-optimal policy with probability  $1/6$ or higher.
  \end{theorem}
  Theorem \ref{thm:lowerbound1} implies that a polynomial dependence on $1/\alpha$ is in general unavoidable
  in the sense that any algorithm either needs to suffer a regret exponential in the horizon $H$, or its regret needs to be polynomially dependent on $1/\alpha$.
The proof of Theorem \ref{thm:lowerbound1} is provided in Appendix~\ref{app:lower}.

\subsection{Overcomplete setting}
\label{sec:over}

We now turn to the more challenging setting of learning in overcomplete POMDPs, where the number of hidden states can be larger than the number of observations. We show that a simple variant of \omle is able to learn weakly-revealing overcomplete POMDPs in a polynomial number of samples. As we shall see, we pay a nontrivial price for the increased generality: while we can still achieve rate-optimal PAC-results, we compromise on the regret of the algorithm.

\paragraph{Algorithm description}
Algorithm \ref{alg:over} shows the pseudo-code of \omle suitable for $m$-step $\alpha$-weakly revealing overcomplete POMDPs. While the basic structure of the method is the same as before, the general \omle, which we call  \emph{multi-step \omle}, differs from the basic version in two important aspects:
\begin{itemize}
  \item Instead of merely following the optimistic policy, Algorithm \ref{alg:over} adopts a more active strategy for exploration. Specifically, for each optimistic policy $\pi^k$, the learner will one by one experiments with $(H-m+1)$ policies that are obtained by picking  a within-episode time index $h\in \{0,\dots,H-m\}$ and then following policy $\pi^k$ for the first $h$ steps,  and then picking actions   uniformly at random in the remaining steps of the episode. We denote the resulting policy by  
$\pi^k_{1:h}\circ\text{uniform}(\fA)$, which abuses notation, but should improve readability.
\item When constructing the confidence set, Algorithm \ref{alg:over} requires the minimum singular value of the  $m$-step emission-action matrix (defined in equation \eqref{defn:M}) to be lower bounded by $\alpha$, which enforces the multi-step \emph{$\alpha$-weakly revealing} condition in Assumption \ref{asp:over}.
\end{itemize}
By trying random action sequences after executing $\pi^k$ for the initial $h$ steps, the learner can gather more information about the hidden states reachable by $\pi^k$ at step $h$ and therefore can better learn the system dynamics under $\pi^k$. The price of trying random actions is that the algorithm as described here will in general have linear regret. Nevertheless, with an online-to-batch conversion, Algorithm \ref{alg:over} serves as a suitable approach to learning a good policy with low sample complexity.

\begin{algorithm}[t]
  \caption{\textsc{Multi-step Optimistic Maximum Likelihood Estimation}} \begin{algorithmic}[1]\label{alg:over}
 \STATE \textbf{Initialize:} $\cB^1 = \{ \hat \theta\in\Theta:~\min_h\sigma_{S}(\M_h(\hat\theta))\ge \alpha\}$, $\cD=\{\}$
    \FOR{$k=1,\ldots,K$}
    \STATE $(\theta^k,\pi^k) = \argmax_{\hat\theta\in\cB^k,\pi} V^\pi(\hat\theta)$
    \label{line:over-1}
    \FOR{$h=0,\ldots,H-m$} \label{line:over-2}
  \STATE execute policy $\pi^k_{1:h}\circ\text{uniform}(\fA)$  to collect a trajectory $\tau^{k,h}$ \\
  then add  $(\pi^k_{1:h}\circ\text{uniform}(\fA),\tau^{k,h})$ into $\cD$  \label{line:over-3}
   \ENDFOR
    \STATE update
    \vspace{-3mm}
    \begin{equation}
    \cB^{k+1} = \bigg\{\hat\theta \in \Theta: \sum_{(\pi,\tau)\in\cD} \log \P_{{\hat\theta}}^{\pi} (\tau)
    \ge \max_{ \theta' \in\Theta} \sum_{(\pi,\tau)\in\cD} \log \P^{\pi}_{{\theta'}}(\tau) -\beta \bigg\}\bigcap \cB^1
    \vspace{-3mm}
    \end{equation}  \label{line:over-4}
    \ENDFOR
 \end{algorithmic}
 \end{algorithm}

\paragraph{Theoretical guarantees}
Our main result in this section bounds the total suboptimality of the policies $\pi^1,\dots,\pi^k$ chosen by \omle.
Note that since \omle is not following these policies, the regret of \omle is different (in general, higher) than the total suboptimality.

\begin{theorem}[Total suboptimality of multi-step \omle]\label{thm:over}
There exists an absolute constant $c>0$ such that for any $\delta\in(0,1]$ and $S,A,O,K,H\in\N$, if we choose parameter $\beta$ in Algorithm \ref{alg:over} as $\beta = c\left(H(S^2A+SO)\log(SAOH)+\log(KH/\delta)\right)$, then, for any POMDP 
with $S$ states, $A$ actions, $O$ observations and horizon $H$ and
satisfying Assumption \ref{asp:over},
	with probability at least $1-\delta$, 
  $$  \sum_{t=1}^k \left( V^{\star}-V^{\pi^t} \right)\le \poly(S,A^m,O,H,\alpha^{-1},\log(\delta^{-1}K))\cdot\sqrt{k}  \qquad \text{ for all }k\in[K].$$
\end{theorem}
The specific polynomial dependency is presented in Appendix \ref{app:proof-over}.
This form of the result is preferred as it makes a comparison to Theorem~\ref{thm:under} more direct and it also reveals a bit of the proof strategy.
The significance of this result is that, 
using the standard online-to-batch conversion \citep{cesa2004generalization}, we get the following sample complexity results.
\begin{corollary}[Sample Complexity of multi-step \omle]
Under the same setting as Theorem \ref{thm:over}, when $K \ge \poly(S,A^m,O,H,\alpha^{-1},\log(\epsilon^{-1}\delta^{-1}))\cdot\epsilon^{-2}$, with probability at least $1-\delta$, the uniform mixture of the policies produced by multi-step \omle is $\epsilon$-optimal. I.e., $(1/K)\cdot\sum_{k=1}^K V^{\pi^k} \ge V^\star - \epsilon$.
\end{corollary}
Up to polylogarithmic factors, the dependence on $\epsilon$ in this result is unimprovable.
Using an explore-then-exploit strategy, this latter result gives rise to a method that enjoys $\tlO(K^{2/3})$ regret, where the constants hidden are still polynomial in the relevant quantities.
To our knowledge, for small fixed $m$, this is the first sample-efficient result for learning overcomplete POMDPs in the \emph{exploration} setting where the algorithm needs to reason about how to collect information efficiently.

A natural question here is whether the exponential dependence on $m$ in Theorem \ref{thm:over} is necessary.
We answer this question by providing the following lower bound, which rules out the possibility of an upper bound polynomial in $m$.
\begin{theorem}[Necessity of $A^{\Omega(m)}$ dependency]\label{thm:lowerbound2}
  For any $m,A\in\N^+$, there exists a  POMDP with $S,H,O=\cO(m)$ and satisfying Assumption \ref{asp:over} with $\alpha \ge 1$ so that any algorithm requires at least $\Omega(A^{m-1})$ samples to learn a $(1/2)$-optimal policy with probability at least $1/2$.
  \end{theorem}


\section{Proof Overview}

We provide a proof overview of Theorem \ref{thm:under} for learning undercomplete \emph{weakly revealing} POMDPs (Assumption \ref{asp:under}). 
We defer the full proof to Appendix \ref{app:proof-under}.
The proof for learning overcomplete POMDPs (Theorem \ref{thm:over}) follows a similar strategy, which is described in Appendix \ref{app:proof-over}.

\subsection{Observable operator models}
\label{subsec:oom}
To begin with, we introduce the observable operators \citep{jaeger2000observable} that provide an alternate parameterization of POMDPs. 
These operators will serve as  intermediate quantities in our analysis: They will allow us to bound the suboptimality of the learned policies as a function of the ``width'' of the MLE confidence set.
Given the transition matrices $\{\T_{h,a}\}_{(h,a)\in[H]\times\fA}$, the observation matrices $\{\O_h\}_{h\in[H]}$, and the initial distribution $\mu_1$, the observable operators  $\{\B_h(o,a)\}_{(h,o,a)\in[H-1]\times\fO\times\fA}$ and the initial $\b_0$ observation distribution are given by 
\begin{equation*}\label{eq:ps-2}
        \B_h(o,a) = \O_{h+1}\T_{h,a}  \diag(\O_h(o\mid \cdot))\O_h^\dagger,\qquad
        \b_0 = \O_1\mu_1,
\end{equation*}
where $\O_h(o\mid \cdot)\in\R^S$ denotes the $o^{\rm th}$ row of $\O_h$. 
It is known that the these operators give an equivalent parameterization of the POMDPs: For any policy, the distribution induced by a POMDP over the possible trajectories of observation-action pairs can be described solely using these operators.
In particular,
the probability of observing trajectory $\tau_h = (o_1,a_1,\ldots,o_h,a_h)$ under policy $\pi$ in POMDP model $\theta$ is given by 
\begin{equation}\label{eq:ps-3}
    \P_\theta^\pi(\tau_h) = \pi(\tau_h)\cdot
     \left(\e_{o_h}\trans\B_{h-1}(o_{h-1},a_{h-1};\theta)\cdots\B_1(o_1,a_1;\theta)\b_0(\theta)\right),
\end{equation}
where $\pi(\tau_h) := \prod_{h'=1}^h \pi(a_h\mid o_h,\tau_{h-1})$ represents the part of the probability of $\tau_h$ that can be attributed to the randomness of the policy and
we used $\B_{j}(\cdot;\theta)$ to denote the observable operators underlying $\theta$.
One important advantage of adopting this operator representation of POMDPs is that the linear structure facilitates us to use existing tools from matrix analysis to analyze the  error of operator estimates.

\subsection{Step 1: bound the regret by the error of operator estimates}
By analyzing the relaxed MLE condition, one can prove that the groundtruth POMDP model $\theta^\star$ is contained in  confidence set $\cB^k$ for all $k\in[K]$ with high probability (see Proposition \ref{prop:mle-optimism} in Appendix \ref{app:mle}). 
Therefore, from now on assume that $\theta^\star\in \cap_{k\in [K]}\cB^k$ holds. 
Now, recall that we choose the model estimate and the behavior policy optimistically in Algorithm \ref{alg:under}, i.e., $(\theta^k,\pi^k)=\argmax_{\hat\theta\in\cB^k,\pi} V^\pi_{\hat\theta}$. 
As a result, we have $V^\star = \max_\pi V^\pi_{\theta^\star}\le \max_{\hat\theta\in\cB^k,\pi} V^\pi_{\hat\theta} = V^{\pi^k}_{\theta^k}$ for all $k\in[K]$. From this, we get 
\begin{equation}\label{eq:ps-1}
    \sum_{t=1}^k V^{\star}_{\theta^\star} - V^{\pi^t}_{\theta^\star}
    \le 
    \sum_{t=1}^k V^{\pi^t}_{\theta^t} - V^{\pi^t}_{\theta^\star}
    \le  H \sum_{t=1}^k \sum_{\tau_H   } | \P^{\pi^t}_{{\theta^t}}(\tau_H) - \P^{\pi^t}_{\theta^\star}(\tau_H)|,
\end{equation}
where $\tau_H=(o_1,a_1,\ldots,o_H,a_H)$ denotes a whole trajectory and the second inequality uses the fact that the cumulative reward of each trajectory is bounded by $H$. Therefore, to prove Theorem \ref{alg:under}, it suffices to bound the
total cumulated 
 error in estimating the probability of the individual trajectories, cf.
the RHS of \eqref{eq:ps-1}.
  
By using the OOM representations in \eqref{eq:ps-3}, it turns out that we can bound the RHS of \eqref{eq:ps-1} by the error in estimating each observable operator. To simplify notation, we abbreviate $\B_h(o,a;\theta^\star)$, $\b_0(\theta^\star)$ as $\B_h(o,a)$, $\b_0$, and denote $\B^t_h(o,a) :=\B_h(o,a;\theta^t)$, $\b_0^t := \b_0(\theta^t)$. With this notation, we have the following result:
\begin{lemma}\label{lem:prod-triangle}
    For any $k\in\N$, the RHS of \eqref{eq:ps-1} is upper bounded by 
    \begin{equation}\label{eq:ps-5}
      \frac{H\sqrt{S}}{\alpha}  \left(\sum_{t=1}^k  \sum_{h=1}^{H-1} \sum_{\tau_{h}   } 
       \left\| \left(\B_h(o_h,a_h) - \B^t_h(o_h,a_h)\right)\b(\tau_{h-1})\right\|_1 \times \pi^t(\tau_{h}) + \|\b_0-\b^t_0\|_1\right),
       \end{equation}
       where $\b(\tau_h) \defeq 
        \left(\prod_{h'=1}^h \B_{h'}(o_{h'},a_{h'}) \right)\b_0$ is the ``belief vector'' associated with trajectory $\tau_h=(o_1,a_1,\dots,o_h,a_h)$.
    \end{lemma}
In Equation~\eqref{eq:ps-5} we abused notation in a few ways: In the innermost sum over the observation-action trajectories $\tau_h$ of length $h$, $\tau_{h-1}$ refers to the prefix of $\tau_h$ where the last observation-action is dropped. Also, in this sum, $o_h,a_h$ refer to the last observation-action pair of $\tau_h$.

Lemma \ref{lem:prod-triangle} is obtained from Lemma \ref{lem:prod-triangle-under},
which states the same result for an arbitrary sequence of observable operators.
 As a result, in order to control the regret, it suffices to control the  estimation error of each operator. 
 Importantly, here we do not need to recover the operators accurately  at all entries, which,  in general, is also  impossible when there are hard-to-reach latent states. 
 Instead, we only care about the projections of the errors onto the belief vectors, which are  further reweighted by the probability of the behavior policies. 
 Therefore, it suffices to learn the operators accurately only in those  directions that are adequately covered by the reweighted belief vectors.

\subsection{Step 2: derive  constraints for the operator estimates from \omle}

Now let us make a detour to see what guarantees \omle~can provide for our operator estimates. 
As a result of the classic MLE analysis \cite[e.g.,][]{geer2000empirical}, we can show under the same choice of  $\beta$ as  Theorem \ref{thm:under}, with high probability
\begin{equation}\label{eq:ps-6}
     \sum_{t=1}^{k-1} \left\| \P^{\pi^t}_{\theta^k}(\tau_h=\cdot) - \P^{\pi^t}_{\theta^\star}(\tau_h=\cdot)\right\|^2_1 = \Ocal( \beta) \quad \text{ for all }(k,h)\in[K]\times[H]. 
\end{equation}
(Proposition \ref{prop:mle-valid} in Appendix \ref{app:mle} gives the precise result.)
In brief, this means the model estimate in the $k^{\rm th}$ iteration, that is  $\theta^k$, can be used to predict the behavior of the policies followed \emph{before} the $k^{\rm th}$ iteration to a certain accuracy. To proceed, we represent the probabilities in equation \eqref{eq:ps-6} by products of operators using  equation \eqref{eq:ps-3} and perform  further algebraic transformations, which   eventually leads to the following lemma for our operator estimates. The proof of this lemma is given in Appendix \ref{subsec:step2}.
\begin{lemma}
Suppose the relation in equation \eqref{eq:ps-6} holds, then 
for all $(k,h)\in[K]\times[H]$ 
\begin{equation}\label{eq:ps-7}
    \sum_{t=1}^{k-1}  \sum_{ \tau_h}   
\left\| \left(\B_h(o_h,a_h) - \B^k_h(o_h,a_h)\right)\b(\tau_{h-1})\right\|_1 \times \pi^t(\tau_h) = \cO\left( \frac{\sqrt{S\beta k}}{\alpha} \right).
\end{equation}
\end{lemma}
Intuitively, the constraints above imply the operator estimates in the $k^{\rm th}$ iteration are close to the true operators when being projected onto the  belief vectors that are reweighted by the historical policies.    However, a careful examination shows that \eqref{eq:ps-7} cannot be directly used to control \eqref{eq:ps-5} because \eqref{eq:ps-5}  involves the operator error of $\theta^t$ reweighted by $\pi^t$ that is the behavior policy in the \emph{same} iteration. This is very different from \eqref{eq:ps-7}. We deal with this problem in Step 3.

\subsection{Step 3: bridge Step $1$ and $2$ via $\ell_1$-norm eluder dimension}
To prove the sample efficiency of optimistic algorithms, one needs to argue that, after a sufficient number of iterations, the size of the maintained confidence set is small enough to guarantee near-optimality of the learned policy. This is typically achieved by resorting to the pigeon-hole principle in the tabular setting \citep[e.g.,][]{azar2017minimax,jin2018q}, or to the elliptical potential lemma in the linear setting \citep[e.g.,][]{lattimore2020bandit}. 

In the context of this paper, by further algebraic transformations, we  reduce the problem of bounding  \eqref{eq:ps-5} by \eqref{eq:ps-7} to proving the following algebraic inequality, which plays a similar role as the elliptical potential lemma. The full inequality is more involved (see Proposition \ref{prop:linear-regret} in Appendix \ref{app:eluder}); here we present a simplified version for the sake of simplicity.
\begin{proposition}\label{prop:ps}
     Suppose sequences $\{w_{k,j}\}_{(k,j)\in[K]\times[m]}$  and $\{x_{k,i}\}_{(k,i)\in[K]\times [n]}$ satisfy that $w_{k,j}, x_{k,i} \in \R^d$ for all $(k,i,j)\in[K]\times [n]\times[m]$. Suppose that we further have
     $$ \sum_{t=1}^{k-1} \sum_{j=1}^m\sum_{i=1}^n| w_{k,j}\trans x_{t,i}| \le \sqrt{k},~~ \sum_{j=1}^m \| w_{k,j} \|_2 \le 1 ~~\text{and}~~
    \sum_{i=1}^n \| x_{k,i} \|_2 \le 1
    \quad  \text{ for all } k\in[K]. $$
    Then we have         
    $\sum_{t=1}^{k} \sum_{j=1}^m\sum_{i=1}^n
    |w_{t,j} \trans x_{t,i}| =\tlO(\sqrt{\zeta k} )$ for all  $k\in[K]$,
where $\zeta$ is a parameter that depends on $d$ only. 
\end{proposition}
At a high level, the precondition and the target in Proposition \ref{prop:ps} correspond to equation  \eqref{eq:ps-7} and  \eqref{eq:ps-5}, respectively (see Appendix \ref{subsec:step3} for details). In the special case of $m=n=1$, Proposition \ref{prop:ps} reduces to
\begin{equation} \label{eq:our_condition}
\text{if }  \sum_{t=1}^{k-1} 
| w_k\trans x_{t}| \le \sqrt{k}  \text{ for all } k\in[K], \quad \text{ then } \sum_{t=1}^{k} 
|w_t \trans x_{t}| =\tilde{\cO}(\sqrt{\zeta k})\text{ for all } k\in[K].
\end{equation}
We compare this with the standard elliptical potential lemma in linear bandit literature \citep[e.g.,][]{lattimore2020bandit},  which is typically of the form: 
\begin{equation} \label{eq:elliptical}
\text{if }  \sum_{t=1}^{k-1} 
| w_k\trans x_{t}|^2 \le 1  \text{ for all } k\in[K], \quad \text{ then } \sum_{t=1}^{k} 
|w_t \trans x_{t}| =\tilde{\cO}(\sqrt{dk})\text{ for all } k\in[K].
\end{equation}
We remark that the precondition in \eqref{eq:elliptical} directly implies the precondition in \eqref{eq:our_condition} by the Cauchy-Swartz inequality. That is, Proposition \ref{prop:ps} is stronger than the standard elliptical potential lemma, and we need to develop new techniques to prove Proposition \ref{prop:ps}.

Noting the close relation between the elliptical potential lemma and the framework of eluder dimension \citep{russo2013eluder} (in its original $\ell_2$-norm form), we develop a new framework based on the $\ell_1$-norm counterpart of eluder dimension, and adapt corresponding techniques to prove that in \eqref{eq:our_condition} and Proposition \ref{prop:ps} we can allow the choice of $\zeta = d^2$ which is one $d$ factor worse than the standard elliptical potential lemma. We defer the details of this framework to Appendix \ref{app:eluder}.

\section{Conclusion}

In this paper, we identified a new rich class of POMDPs, which we call \emph{weakly revealing} POMDPs. 
\emph{Weakly revealing} POMDPs subsume a majority of existing POMDPs that are known to be sample-efficiently learnable, and include both undercomplete and overcomplete POMDPs. 
We further propose a new simple algorithm, \omle, which combines optimism with maximum likelihood estimation.
We prove that \omle~can learn a near-optimal policy for any weakly revealing POMDP using polynomial samples. 
We complement our positive results with two lower bounds to justify the necessity of the appearance of certain problem-dependent quantities in our upper bounds. Finally, while our work shows that sample-efficient learning is possible in large classes of POMDPs, computationally efficient learning of POMDPs remains challenging, which we leave for future work. 

\section*{Acknowledgement}
We thank Nan Jiang for valuable discussions on the sample complexity of \emph{multi-step \omle}.




\bibliographystyle{unsrtnat}
\bibliography{ref}

\begin{thebibliography}{54}
\providecommand{\natexlab}[1]{#1}
\providecommand{\url}[1]{\texttt{#1}}
\expandafter\ifx\csname urlstyle\endcsname\relax
  \providecommand{\doi}[1]{doi: #1}\else
  \providecommand{\doi}{doi: \begingroup \urlstyle{rm}\Url}\fi

\bibitem[Akkaya et~al.(2019)Akkaya, Andrychowicz, Chociej, Litwin, McGrew,
  Petron, Paino, Plappert, Powell, Ribas, et~al.]{akkaya2019solving}
Ilge Akkaya, Marcin Andrychowicz, Maciek Chociej, Mateusz Litwin, Bob McGrew,
  Arthur Petron, Alex Paino, Matthias Plappert, Glenn Powell, Raphael Ribas,
  et~al.
\newblock Solving {R}ubik's cube with a robot hand.
\newblock \emph{arXiv preprint arXiv:1910.07113}, 2019.

\bibitem[Vinyals et~al.(2019)Vinyals, Babuschkin, Czarnecki, Mathieu, Dudzik,
  Chung, Choi, Powell, Ewalds, Georgiev, et~al.]{vinyals2019grandmaster}
Oriol Vinyals, Igor Babuschkin, Wojciech~M Czarnecki, Michael Mathieu, Andrew
  Dudzik, Junyoung Chung, David~H Choi, Richard Powell, Timo Ewalds, Petko
  Georgiev, et~al.
\newblock Grandmaster level in {StarCraft II} using multi-agent reinforcement
  learning.
\newblock \emph{Nature}, 575\penalty0 (7782):\penalty0 350--354, 2019.

\bibitem[Brown and Sandholm(2019)]{brown2019superhuman}
Noam Brown and Tuomas Sandholm.
\newblock Superhuman {{AI}} for multiplayer poker.
\newblock \emph{Science}, 365\penalty0 (6456):\penalty0 885--890, 2019.

\bibitem[Levinson et~al.(2011)Levinson, Askeland, Becker, Dolson, Held, Kammel,
  Kolter, Langer, Pink, Pratt, et~al.]{levinson2011towards}
Jesse Levinson, Jake Askeland, Jan Becker, Jennifer Dolson, David Held, Soeren
  Kammel, J~Zico Kolter, Dirk Langer, Oliver Pink, Vaughan Pratt, et~al.
\newblock Towards fully autonomous driving: Systems and algorithms.
\newblock In \emph{2011 IEEE intelligent vehicles symposium (IV)}, pages
  163--168. IEEE, 2011.

\bibitem[Bower and Gilbert(2005)]{bower2005resource}
Joseph~L Bower and Clark~G Gilbert.
\newblock \emph{From resource allocation to strategy}.
\newblock Oxford University Press, 2005.

\bibitem[Hauskrecht and Fraser(2000)]{hauskrecht2000planning}
Milos Hauskrecht and Hamish Fraser.
\newblock Planning treatment of ischemic heart disease with partially
  observable {{M}arkov} decision processes.
\newblock \emph{Artificial Intelligence in Medicine}, 18\penalty0 (3):\penalty0
  221--244, 2000.

\bibitem[Li et~al.(2010)Li, Chu, Langford, and Schapire]{li2010contextual}
Lihong Li, Wei Chu, John Langford, and Robert~E Schapire.
\newblock A contextual-bandit approach to personalized news article
  recommendation.
\newblock In \emph{Proceedings of the 19th international conference on World
  wide web}, pages 661--670, 2010.

\bibitem[De~Brito and Van Der~Laan(2009)]{de2009inventory}
Marisa~P De~Brito and Erwin~A Van Der~Laan.
\newblock Inventory control with product returns: The impact of imperfect
  information.
\newblock \emph{European journal of operational research}, 194\penalty0
  (1):\penalty0 85--101, 2009.

\bibitem[Cassandra et~al.(1996)Cassandra, Kaelbling, and
  Kurien]{cassandra1996acting}
Anthony~R Cassandra, Leslie~Pack Kaelbling, and James~A Kurien.
\newblock Acting under uncertainty: Discrete {B}ayesian models for mobile-robot
  navigation.
\newblock In \emph{Proceedings of IEEE/RSJ International Conference on
  Intelligent Robots and Systems. IROS'96}, volume~2, pages 963--972. IEEE,
  1996.

\bibitem[Papadimitriou and Tsitsiklis(1987)]{papadimitriou1987complexity}
Christos~H Papadimitriou and John~N Tsitsiklis.
\newblock The complexity of {{M}arkov} decision processes.
\newblock \emph{Mathematics of operations research}, 12\penalty0 (3):\penalty0
  441--450, 1987.

\bibitem[Mundhenk et~al.(2000)Mundhenk, Goldsmith, Lusena, and
  Allender]{mundhenk2000complexity}
Martin Mundhenk, Judy Goldsmith, Christopher Lusena, and Eric Allender.
\newblock Complexity of finite-horizon {M}arkov decision process problems.
\newblock \emph{Journal of the ACM (JACM)}, 47\penalty0 (4):\penalty0 681--720,
  2000.

\bibitem[Vlassis et~al.(2012)Vlassis, Littman, and
  Barber]{vlassis2012computational}
Nikos Vlassis, Michael~L Littman, and David Barber.
\newblock On the computational complexity of stochastic controller optimization
  in {{PO{MDP}}s}.
\newblock \emph{ACM Transactions on Computation Theory (TOCT)}, 4\penalty0
  (4):\penalty0 1--8, 2012.

\bibitem[Mossel and Roch(2005)]{mossel2005learning}
Elchanan Mossel and S{\'e}bastien Roch.
\newblock Learning nonsingular phylogenies and hidden {{M}arkov} models.
\newblock In \emph{Proceedings of the thirty-seventh annual ACM symposium on
  Theory of computing}, pages 366--375, 2005.

\bibitem[Kwon et~al.(2021{\natexlab{a}})Kwon, Efroni, Caramanis, and
  Mannor]{kwon2021rl}
Jeongyeol Kwon, Yonathan Efroni, Constantine Caramanis, and Shie Mannor.
\newblock {{RL}} for latent {{MDP}}s: Regret guarantees and a lower bound.
\newblock \emph{Advances in Neural Information Processing Systems}, 34,
  2021{\natexlab{a}}.

\bibitem[Kwon et~al.(2021{\natexlab{b}})Kwon, Efroni, Caramanis, and
  Mannor]{kwon2021reinforcement}
Jeongyeol Kwon, Yonathan Efroni, Constantine Caramanis, and Shie Mannor.
\newblock Reinforcement learning in reward-mixing {{MDP}s}.
\newblock \emph{Advances in Neural Information Processing Systems}, 34,
  2021{\natexlab{b}}.

\bibitem[Azizzadenesheli et~al.(2016)Azizzadenesheli, Lazaric, and
  Anandkumar]{azizzadenesheli2016reinforcement}
Kamyar Azizzadenesheli, Alessandro Lazaric, and Animashree Anandkumar.
\newblock Reinforcement learning of {PO{MDP}}s using spectral methods.
\newblock In \emph{Conference on Learning Theory}, pages 193--256. PMLR, 2016.

\bibitem[Guo et~al.(2016)Guo, Doroudi, and Brunskill]{guo2016pac}
Zhaohan~Daniel Guo, Shayan Doroudi, and Emma Brunskill.
\newblock A {{PAC} {RL}} algorithm for episodic {{PO{MDP}}s}.
\newblock In \emph{Artificial Intelligence and Statistics}, pages 510--518.
  PMLR, 2016.

\bibitem[Jin et~al.(2020{\natexlab{a}})Jin, Kakade, Krishnamurthy, and
  Liu]{jin2020sample}
Chi Jin, Sham~M Kakade, Akshay Krishnamurthy, and Qinghua Liu.
\newblock Sample-efficient reinforcement learning of undercomplete
  {{PO{MDP}}s}.
\newblock \emph{NeurIPS}, 2020{\natexlab{a}}.

\bibitem[Wilks(1938)]{wilks1938large}
Samuel~S Wilks.
\newblock The large-sample distribution of the likelihood ratio for testing
  composite hypotheses.
\newblock \emph{The annals of mathematical statistics}, 9\penalty0
  (1):\penalty0 60--62, 1938.

\bibitem[Jaeger(2000)]{jaeger2000observable}
Herbert Jaeger.
\newblock Observable operator models for discrete stochastic time series.
\newblock \emph{Neural computation}, 12\penalty0 (6):\penalty0 1371--1398,
  2000.

\bibitem[Geer et~al.(2000)Geer, van~de Geer, and Williams]{geer2000empirical}
Sara~A Geer, Sara van~de Geer, and D~Williams.
\newblock \emph{Empirical Processes in M-estimation}, volume~6.
\newblock Cambridge University Press, 2000.

\bibitem[Azar et~al.(2017)Azar, Osband, and Munos]{azar2017minimax}
Mohammad~Gheshlaghi Azar, Ian Osband, and R{\'e}mi Munos.
\newblock Minimax regret bounds for reinforcement learning.
\newblock In \emph{International Conference on Machine Learning}, pages
  263--272. PMLR, 2017.

\bibitem[Jin et~al.(2018)Jin, Allen-Zhu, Bubeck, and Jordan]{jin2018q}
Chi Jin, Zeyuan Allen-Zhu, Sebastien Bubeck, and Michael~I Jordan.
\newblock Is {Q}-learning provably efficient?
\newblock \emph{Advances in neural information processing systems}, 31, 2018.

\bibitem[Lattimore and Szepesv{\'a}ri(2020)]{lattimore2020bandit}
Tor Lattimore and Csaba Szepesv{\'a}ri.
\newblock \emph{Bandit algorithms}.
\newblock Cambridge University Press, 2020.

\bibitem[Russo and Van~Roy(2013)]{russo2013eluder}
Daniel Russo and Benjamin Van~Roy.
\newblock Eluder dimension and the sample complexity of optimistic exploration.
\newblock In \emph{NIPS}, pages 2256--2264. Citeseer, 2013.

\bibitem[Dann et~al.(2017)Dann, Lattimore, and Brunskill]{dann2017unifying}
Christoph Dann, Tor Lattimore, and Emma Brunskill.
\newblock Unifying {{PAC}} and regret: Uniform {{PAC}} bounds for episodic
  reinforcement learning.
\newblock \emph{Advances in Neural Information Processing Systems}, 30, 2017.

\bibitem[Jin et~al.(2020{\natexlab{b}})Jin, Yang, Wang, and
  Jordan]{jin2020provably}
Chi Jin, Zhuoran Yang, Zhaoran Wang, and Michael~I Jordan.
\newblock Provably efficient reinforcement learning with linear function
  approximation.
\newblock In \emph{Conference on Learning Theory}, pages 2137--2143. PMLR,
  2020{\natexlab{b}}.

\bibitem[Zanette et~al.(2020)Zanette, Lazaric, Kochenderfer, and
  Brunskill]{zanette2020learning}
Andrea Zanette, Alessandro Lazaric, Mykel Kochenderfer, and Emma Brunskill.
\newblock Learning near optimal policies with low inherent {B}ellman error.
\newblock In \emph{International Conference on Machine Learning}, pages
  10978--10989. PMLR, 2020.

\bibitem[Jiang et~al.(2017)Jiang, Krishnamurthy, Agarwal, Langford, and
  Schapire]{jiang2017contextual}
Nan Jiang, Akshay Krishnamurthy, Alekh Agarwal, John Langford, and Robert~E
  Schapire.
\newblock Contextual decision processes with low {{B}ellman} rank are
  {{PAC}}-learnable.
\newblock In \emph{International Conference on Machine Learning}, pages
  1704--1713. PMLR, 2017.

\bibitem[Krishnamurthy et~al.(2016)Krishnamurthy, Agarwal, and
  Langford]{krishnamurthy2016pac}
Akshay Krishnamurthy, Alekh Agarwal, and John Langford.
\newblock {PAC} reinforcement learning with rich observations.
\newblock \emph{Advances in Neural Information Processing Systems}, 29, 2016.

\bibitem[Even-Dar et~al.(2005)Even-Dar, Kakade, and
  Mansour]{even2005reinforcement}
Eyal Even-Dar, Sham~M Kakade, and Yishay Mansour.
\newblock Reinforcement learning in {PO{MDP}}s without resets.
\newblock 2005.

\bibitem[Poupart and Vlassis(2008)]{poupart2008model}
Pascal Poupart and Nikos Vlassis.
\newblock Model-based {B}ayesian reinforcement learning in partially observable
  domains.
\newblock In \emph{Proc Int. Symp. on Artificial Intelligence and Mathematics},
  pages 1--2, 2008.

\bibitem[Ross et~al.(2007)Ross, Chaib-draa, and Pineau]{ross2007bayes}
Stephane Ross, Brahim Chaib-draa, and Joelle Pineau.
\newblock {B}ayes-adaptive {{PO{MDP}}s}.
\newblock \emph{Advances in neural information processing systems}, 20, 2007.

\bibitem[Azizzadenesheli et~al.(2018)Azizzadenesheli, Yue, and
  Anandkumar]{azizzadenesheli2018policy}
Kamyar Azizzadenesheli, Yisong Yue, and Animashree Anandkumar.
\newblock Policy gradient in partially observable environments: Approximation
  and convergence.
\newblock \emph{arXiv preprint arXiv:1810.07900}, 2018.

\bibitem[Xiong et~al.(2021)Xiong, Chen, Gao, and Zhou]{xiong2021sublinear}
Yi~Xiong, Ningyuan Chen, Xuefeng Gao, and Xiang Zhou.
\newblock Sublinear regret for learning {{PO{MDP}}s}.
\newblock \emph{arXiv preprint arXiv:2107.03635}, 2021.

\bibitem[Hsu et~al.(2012)Hsu, Kakade, and Zhang]{hsu2012spectral}
Daniel Hsu, Sham~M Kakade, and Tong Zhang.
\newblock A spectral algorithm for learning hidden {{M}arkov} models.
\newblock \emph{Journal of Computer and System Sciences}, 78\penalty0
  (5):\penalty0 1460--1480, 2012.

\bibitem[Anandkumar et~al.(2014)Anandkumar, Ge, Hsu, Kakade, and
  Telgarsky]{anandkumar2014tensor}
Animashree Anandkumar, Rong Ge, Daniel Hsu, Sham~M Kakade, and Matus Telgarsky.
\newblock Tensor decompositions for learning latent variable models.
\newblock \emph{Journal of machine learning research}, 15:\penalty0 2773--2832,
  2014.

\bibitem[Jafarnia-Jahromi et~al.(2021)Jafarnia-Jahromi, Jain, and
  Nayyar]{jafarnia2021online}
Mehdi Jafarnia-Jahromi, Rahul Jain, and Ashutosh Nayyar.
\newblock Online learning for unknown partially observable mdps.
\newblock \emph{arXiv preprint arXiv:2102.12661}, 2021.

\bibitem[Golowich et~al.(2022)Golowich, Moitra, and
  Rohatgi]{golowich2022planning}
Noah Golowich, Ankur Moitra, and Dhruv Rohatgi.
\newblock Planning in observable {POMDPs} in quasipolynomial time.
\newblock \emph{arXiv preprint arXiv:2201.04735}, 2022.

\bibitem[Efroni et~al.(2022)Efroni, Jin, Krishnamurthy, and
  Miryoosefi]{efroni2022provable}
Yonathan Efroni, Chi Jin, Akshay Krishnamurthy, and Sobhan Miryoosefi.
\newblock Provable reinforcement learning with a short-term memory.
\newblock \emph{arXiv preprint arXiv:2202.03983}, 2022.

\bibitem[Du et~al.(2019)Du, Krishnamurthy, Jiang, Agarwal, Dudik, and
  Langford]{du2019provably}
Simon Du, Akshay Krishnamurthy, Nan Jiang, Alekh Agarwal, Miroslav Dudik, and
  John Langford.
\newblock Provably efficient {RL} with rich observations via latent state
  decoding.
\newblock In \emph{International Conference on Machine Learning}, pages
  1665--1674. PMLR, 2019.

\bibitem[Misra et~al.(2020)Misra, Henaff, Krishnamurthy, and
  Langford]{misra2020kinematic}
Dipendra Misra, Mikael Henaff, Akshay Krishnamurthy, and John Langford.
\newblock Kinematic state abstraction and provably efficient rich-observation
  reinforcement learning.
\newblock In \emph{International conference on machine learning}, pages
  6961--6971. PMLR, 2020.

\bibitem[Ayoub et~al.(2020)Ayoub, Jia, Szepesvari, Wang, and
  Yang]{ayoub2020model}
Alex Ayoub, Zeyu Jia, Csaba Szepesvari, Mengdi Wang, and Lin Yang.
\newblock Model-based reinforcement learning with value-targeted regression.
\newblock In \emph{International Conference on Machine Learning}, pages
  463--474. PMLR, 2020.

\bibitem[Du et~al.(2021)Du, Kakade, Lee, Lovett, Mahajan, Sun, and
  Wang]{du2021bilinear}
Simon Du, Sham Kakade, Jason Lee, Shachar Lovett, Gaurav Mahajan, Wen Sun, and
  Ruosong Wang.
\newblock Bilinear classes: A structural framework for provable generalization
  in {{RL}}.
\newblock In \emph{International Conference on Machine Learning}, pages
  2826--2836. PMLR, 2021.

\bibitem[Jin et~al.(2021)Jin, Liu, and Miryoosefi]{jin2021bellman}
Chi Jin, Qinghua Liu, and Sobhan Miryoosefi.
\newblock {B}ellman eluder dimension: New rich classes of {{RL}} problems, and
  sample-efficient algorithms.
\newblock \emph{Advances in Neural Information Processing Systems}, 34, 2021.

\bibitem[Foster et~al.(2021)Foster, Kakade, Qian, and
  Rakhlin]{foster2021statistical}
Dylan~J Foster, Sham~M Kakade, Jian Qian, and Alexander Rakhlin.
\newblock The statistical complexity of interactive decision making.
\newblock \emph{arXiv preprint arXiv:2112.13487}, 2021.

\bibitem[Lai(1987)]{lai1987adaptive}
Tze~Leung Lai.
\newblock Adaptive treatment allocation and the multi-armed bandit problem.
\newblock \emph{The Annals of Statistics}, pages 1091--1114, 1987.

\bibitem[Kumar and Becker(1982)]{kumar1982new}
P~Kumar and A~Becker.
\newblock A new family of optimal adaptive controllers for {Markov} chains.
\newblock \emph{IEEE Transactions on Automatic Control}, 27\penalty0
  (1):\penalty0 137--146, 1982.

\bibitem[Mete et~al.(2021)Mete, Singh, Liu, and Kumar]{mete2021reward}
Akshay Mete, Rahul Singh, Xi~Liu, and PR~Kumar.
\newblock Reward biased maximum likelihood estimation for reinforcement
  learning.
\newblock In \emph{Learning for Dynamics and Control}, pages 815--827. PMLR,
  2021.

\bibitem[Agarwal et~al.(2020)Agarwal, Kakade, Krishnamurthy, and
  Sun]{agarwal2020flambe}
Alekh Agarwal, Sham Kakade, Akshay Krishnamurthy, and Wen Sun.
\newblock Flambe: Structural complexity and representation learning of low rank
  mdps.
\newblock \emph{NeurIPS}, 2020.

\bibitem[Uehara et~al.(2021)Uehara, Zhang, and Sun]{uehara2021representation}
Masatoshi Uehara, Xuezhou Zhang, and Wen Sun.
\newblock Representation learning for online and offline {RL} in low-rank
  {MDPs}.
\newblock In \emph{International Conference on Learning Representations}, 2021.

\bibitem[Auer et~al.(1995)Auer, Cesa-Bianchi, Freund, and
  Schapire]{auer1995gambling}
Peter Auer, Nicolo Cesa-Bianchi, Yoav Freund, and Robert~E Schapire.
\newblock Gambling in a rigged casino: The adversarial multi-armed bandit
  problem.
\newblock In \emph{Proceedings of IEEE 36th annual foundations of computer
  science}, pages 322--331. IEEE, 1995.

\bibitem[Cesa-Bianchi et~al.(2004)Cesa-Bianchi, Conconi, and
  Gentile]{cesa2004generalization}
Nicolo Cesa-Bianchi, Alex Conconi, and Claudio Gentile.
\newblock On the generalization ability of online learning algorithms.
\newblock \emph{IEEE Transactions on Information Theory}, 50\penalty0
  (9):\penalty0 2050--2057, 2004.

\bibitem[Zhang(2006)]{zhang2006}
Tong Zhang.
\newblock From $\epsilon$-entropy to {KL}-entropy: Analysis of minimum
  information complexity density estimation.
\newblock \emph{The Annals of Statistics}, 34\penalty0 (5):\penalty0
  2180--2210, 2006.

\end{thebibliography}
\clearpage
\newpage

\appendix

\section{Maximum Likelihood Estimation}
\label{app:mle}

In this section, we analyze the maximum likelihood estimation (MLE) approach for the following meta-algorithm.
Since Algorithm \ref{alg:under} and \ref{alg:over} can be viewed as special cases of Algorithm \ref{alg:meta}, all the results developed in this section directly apply to their analysis. 

\begin{algorithm}[H]
    \caption{Meta-algorithm}\label{alg:meta}
 \begin{algorithmic}
    \FOR{$t=1,\ldots,T$}
    \STATE choose policy $\pi^t$ as a deterministic function of  $\{(\pi^i,\tau^i)\}_{i=1}^{t-1}$
    \STATE execute policy $\pi^t$ and collect a trajectory $\tau^t$
    \ENDFOR
 \end{algorithmic}
 \end{algorithm}

For the reader's convenience, we recall the definitions of the following notations: (a) $\theta=(\T,\O,\mu)$ denotes the ensemble of all the parameters of a POMDP model, (b) $\Theta$ denotes the collections of all such POMDP parameter ensembles, and (c) $\theta^\star$ denotes the parameter ensemble of the groundtruth POMDP that we are interacting with.
We will view $\Theta$ as a subset of a Euclidean space; in particular, $\Theta \subset \R^{H(S^2A+SO)+S}$, so $\theta,\theta^\star \in  \R^{H(S^2A+SO)+S}$.

The first proposition shows that,
up to certain error, with high probability,
the log-likelihood of the groundtruth model computed using  the historical data is close to the maximum log-likelihood.

\begin{proposition}\label{prop:mle-optimism}
There exists an absolute constant $c$ such that for any $\delta\in(0,1]$, with probability at least $1-\delta$: the following inequality holds for all $t\in[T]$ and \textbf{all} $\theta\in\Theta$
	\begin{equation}
		\sum_{i=1}^{t} \log\left(\frac{\P^{\pi^i}_{{\theta}}(\tau ^i)}{\P^{\pi^i}_{{\theta}^\star}(\tau ^i)}\right) \le c\left(H(S^2A+SO)\log(TSAOH)+\log(T/\delta)\right) .
	\end{equation}
\end{proposition}

Our second claim shows that any POMDP model,  whose log-likelihood on the historical data is comparable to that of the groundtruth model, will produce similar distributions of trajectories as the groundtruth model under historical polices.
\begin{proposition}\label{prop:mle-valid}
There exists a universal constant $c$ such that for any  $\delta\in(0,1]$, with probability at least $1-\delta$  for all $t\in[T]$ and \textbf{all} $\theta\in\Theta$, it holds that
	\begin{equation}
		\begin{aligned}
			\MoveEqLeft\sum_{i=1}^{t}\left( \sum_{\tau\in(\fO\times\fA)^H}\left| \P^{\pi^i}_{\theta}(\tau) - \P^{\pi^i}_{\theta^\star}(\tau) \right|\right)^2 \\
			\le &  c \left( \sum_{i=1}^{t} \log\left(\frac{\P^{\pi^i}_{{\theta}^\star}(\tau ^i)}{\P^{\pi^i}_{{\theta}}(\tau ^i)}\right) + H(S^2A+SO)\log(TSAOH)+ \log(T/\delta)\right).
		\end{aligned}
		\label{eq:mle-valid}
	\end{equation}
\end{proposition}
The proofs of Proposition 
\ref{prop:mle-optimism} and \ref{prop:mle-valid} 
can be found in Appendix \ref{app:mle-proof}.


\section{Proofs for Maximum Likelihood Estimation}
\label{app:mle-proof}

In this section, we prove the two propositions stated in Appendix \ref{app:mle}.
For technical purposes, we introduce the concept of optimistic $\epsilon$-discretization. Specifically, we denote $\bar{\theta}$ as the optimistic $\epsilon$-discretization of $\theta$ so that $\bar{\theta}_i = \lceil\theta_i/\epsilon\rceil\times\epsilon$ for all coordinate $i$. We comment that although $\bar \theta$ is not a legal POMDP parameterization, it can still be used to compute the probability of observing any trajectory $\tau$ under any policy $\pi$ by simply replacing $\theta_i$ with $\bar \theta_i$ in the computation. 
In particular, for   
$\theta = (\T,\O,\mu_1)\in \R^d$ with $d = H(S^2A + SO) + S$, given a trajectory $\tau=(o_1,a_1,\dots,o_H,a_H)$ and a policy $\pi$ we define
\begin{align*}
\P^{\pi}_{\theta}(\tau)
& = \sum_{s_1,\dots,s_H\in \fS}
 \mu_1(s_1)  \O_1(o_1|s_1)  \pi_1(a_1|o_1) \T_{1,a_1}(s_2|s_1) \times \\
&  \cdots \\
& \times
 \O_1(o_{H-1}|s_{H-1})  \pi_1(a_{H-1}|o_1,a_1,\dots,o_{H-2},a_{H-2},o_{H-1})  \T_{H-1,a_{H-1}}(s_{H}|s_{H-1})\\
& \times
 \O_1(o_H|s_H)  \pi_1(a_H|o_1,a_1,\dots,o_{H-1},a_{H-1},o_H).
\end{align*}
Note that the right-hand side gives a probability measure over the trajectories when $\theta\in \Theta$, but not in general. To simplify the language, we will still call $\P^{\pi}_{\theta}(\tau)$ the ``probability'' of $\tau$ no matter whether this is indeed a probability.
The significance of using optimistic discretization is that
for any $\theta\in \Theta$ and trajectory $\tau$, we always have that
\begin{align*}
\P^{\pi}_{\bar\theta}(\tau) \ge \P^{\pi}_{\theta}(\tau)\,.
\end{align*}

We denote by $\bar\Theta$ the collections of all such  $\bar\theta$, i.e., $\bar \Theta = \{ \bar \theta: \ \theta \in \Theta\}$.
Throughout this section, we will choose a fixed  $\epsilon$ satisfying that for any $\theta\in\Theta$ and any policy $\pi$, 
\begin{align}
\| \P_\theta^\pi -\P_{\bar\theta}^\pi\|_1 \le 1/T\,.
\label{eq:papprox}
\end{align}
It is not hard to see that one can choose 
$\epsilon\le 1/(C(S+O+A)H T)$ for this condition to be satisfied where $C$ is some large absolute constant.

Since $\bar\theta$ belongs to $[0,1]^d$ with $d=S+H(S^2A+SO)$, the log-cardinality of $\bar\Theta$ defined by this particular $\epsilon$ is at most $\cO(H(S^2 A + SO) \log(TSAOH))$.

\subsection{Proof of Proposition \ref{prop:mle-optimism} }

\begin{proof}The proof is rather standard \citep[e.g., see][]{geer2000empirical} and uses Cram\'er-Chernoff's method.
Pick any $\bar\theta\in\bar\Theta$ and $t\in[T]$.
	Denote $\E_t[\cdot]=\E[\cdot \mid \{(\pi^i,\tau^i)\}_{i=1}^{t-1} \cup \{\pi^t\}]$. We have
	\begingroup
\allowdisplaybreaks
	\begin{align*}
	\MoveEqLeft
	\E \left[\exp \left( \sum_{i=1}^{t} \log\left(\frac{\P^{\pi^i}_{\bar{\theta}}(\tau^i)}{\P^{\pi^i}_{{\theta}^\star}(\tau^i)}\right)\right)\right]\\
		= & \E \left[\exp \left( \sum_{i=1}^{t-1} \log\left(\frac{\P^{\pi^i}_{\bar{\theta}}(\tau^i)}{\P^{\pi^i}_{{\theta}^\star}(\tau^i)}\right)\right) \cdot \E_t\left[\exp \left( \log\left(\frac{\P^{\pi^t}_{\bar{\theta}}(\tau^t)}{\P^{\pi^t}_{{\theta}^\star}(\tau^t)}\right)\right)  \right]  \right]\\
		= & \E \left[\exp \left( \sum_{i=1}^{t-1} \log\left(\frac{\P^{\pi^i}_{\bar{\theta}}(\tau^i)}{\P^{\pi^i}_{{\theta}^\star}(\tau^i)}\right)\right) \cdot \E_t\left[\frac{\P^{\pi^t}_{\bar{\theta}}(\tau^t)}{\P^{\pi^t}_{{\theta}^\star}(\tau^t)} \right]  \right]\\
		= & \E \left[\exp \left( \sum_{i=1}^{t-1} \log\left(\frac{\P^{\pi^i}_{\bar{\theta}}(\tau^i)}{\P^{\pi^i}_{{\theta}^\star}(\tau^i)}\right)\right) \cdot \|\P^{\pi^t}_{\bar{\theta}}(\tau=\cdot)\|_1  \right]\\
			\le &
			\E \left[\exp \left( \sum_{i=1}^{t-1} \log\left(\frac{\P^{\pi^i}_{\bar{\theta}}(\tau^i)}{\P^{\pi^i}_{{\theta}^\star}(\tau^i)}\right)\right) \cdot \left(1+\frac{1}{T}\right) \right]
			\le \cdots \le  e,
	\end{align*}
\endgroup
where the first inequality follows from \eqref{eq:papprox}.
Therefore, by Markov's inequality, we have
	\begin{align*}
		\P \left( \sum_{i=1}^{t} \log\left(\frac{\P^{\pi^i}_{\bar{\theta}}(\tau^i)}{\P^{\pi^i}_{{\theta}^\star}(\tau^i)}\right)  > \log(1/\delta)\right)
		&\le \E \left[\exp \left( \sum_{i=1}^{t} \log\left(\frac{\P^{\pi^i}_{\bar{\theta}}(\tau^i)}{\P^{\pi^i}_{{\theta}^\star}(\tau^i)}\right)\right)\right] \cdot \exp\left[ -\log(1/\delta)\right]\\
		& = e\delta.
	\end{align*}
	Taking a union bound for all $(\bar\theta,t)\in\bar\Theta\times[T]$ and rescaling $\delta$, we obtain 
	\begin{align*}
		\P \left( \max_{(\bar\theta,t)\in\bar\Theta\times [T]} \sum_{i=1}^{t} \log\left(\frac{\P^{\pi^i}_{\bar{\theta}}(\tau^i)}{\P^{\pi^i}_{{\theta}^\star}(\tau^i)}\right)  > c\left(H(S^2A+SO)\log(TSAOH)+\log(T/\delta)\right)\right)
		\le \delta,
	\end{align*}
 where $c>0$ is some absolute constant. 
Finally, recall $\bar\theta$ is an optimistic discretization of $\theta$, which implies
$\P^{\pi}_{{\theta}}(\tau) \le \P^{\pi}_{\bar{\theta}}(\tau)$ for all $\theta,\pi,y$. As a result, we conclude that
\begin{align*}
		\P \left( \max_{(\theta,t)\in\Theta\times [T]} \sum_{i=1}^{t} \log\left(\frac{\P^{\pi^i}_{{\theta}}(\tau^i)}{\P^{\pi^i}_{{\theta}^\star}(\tau^i)}\right)  > c\left(H(S^2A+SO)\log(TSAOH)+\log(T/\delta)\right)\right)
		\le \delta.
	\end{align*}
\end{proof}

\subsection{Proof of Proposition \ref{prop:mle-valid}}

The proof in this section largely follows \citet{agarwal2020flambe}, which is inspired by \citet{zhang2006}.

We start by recalling the meta-algorithm (Algorithm \ref{alg:meta}): in each iteration $t\in[T]$, we pick $\pi^t$ deterministically based on $\{(\pi^i,\tau^i)\}_{i=1}^{t-1}$ and sample $\tau^t$ from  $\P_{\theta^\star}^{\pi^t}$. Now consider a \emph{fixed} $t\in[T]$;
let $D=\{(\pi^i,\tau^i)\}_{i=1}^t$ denote the sequence of policy-trajectory pairs observed within the first $t$ iterations, and denote by $\tilde D$ a ``tangent'' sequence $\{(\pi^i,{\tilde\tau^i})\}_{i=1}^t$ where  ${\tilde\tau^i}\sim \P_{\theta^\star}^{\pi^i}$. Note that $D$ and $\tilde D$ share the same policy parts and their trajectories are independently sampled from the same distributions.

\begin{lemma}\label{lem:aux}
Let $\ell:=\ell(\pi,\tau)$ be a  
real-valued function that maps a policy $\pi$ and a trajectory $\tau$ to $\R$ . 
Let $L(D) = \sum_{i=1}^t \ell (\pi^i,\tau^i)$, 
and $L(\tilde D) = \sum_{i=1}^t \ell (\pi^i,\tilde \tau^i)$.
Then,
$$
\E \left[ \exp\left(L(D)-\log \E[ \exp(L(\tilde D))\mid D] \right) \right] = 1.
$$
\end{lemma}
\begin{proof}
Before proving the lemma,
we remark that all the expectations considered above and below always exist because each $\tau_i$ and $\pi_i$ can only take finitely  many  different values. 
It is direct to see $\tau_i$ only has finitely many possibilities since the number of observations and actions are finite.
As for $\pi_i$, notice that $\pi_i$ is a deterministic function of $\{(\pi^j,\tau^j)\}_{j=1}^{i-1}$ in Algorithm \ref{alg:meta}, so it follows by induction that 
$\pi_i$ also only takes finitely many different values. 

Now we prove Lemma \ref{lem:aux}. Define $E_i = \E[ \exp\left( \ell (\pi^i,\tau^i)\right)  \mid  \pi^i]$
and 
\[
u_t =\E \left[ \exp\left(L(D)-\log \E[ \exp(L(\tilde D))\mid D] \right) \right]
=
 \E \left[   \frac{\exp\left(\sum_{i=1}^t \ell (\pi^i,\tau^i) \right) }
 					    { \E[ \exp(\sum_{i=1}^t \ell (\pi^i,\tilde\tau^i) )\mid D] }
	\right] \,.
\]
Our goal is to show that $u_t=1$.
Owning to the definition of $\{\tilde \tau_i\}_i$, some calculation gives
\begin{align*}
\E\left[ \exp\left(\sum_{i=1}^t \ell (\pi^i,\tilde\tau^i) \right)\mid D\right] 
	= \prod_{i=1}^t E_i\,.
\end{align*}
Plugging this into the last expression obtained for $u_t$ and using the tower rule
with 
\[
D_{t-1}=(\pi^1,\tau^1,\dots,\pi^{t-1},\tau^{t-1},\pi^t)
\]
we get
\begin{align*}
\MoveEqLeft
 u_t
  =  
 \E \left[
 	\E\left[ 
    \frac{  \exp(\sum_{i=1}^t \ell (\pi^i,\tau^i)) }
 					    { \prod_{i=1}^t E_i }
	\mid D_{t-1}\right]
	\right] 
  =  
 \E \left[
    \frac{ \E[ \exp(\sum_{i=1}^t \ell (\pi^i,\tau^i))\mid D_{t-1}] }
 					    { \prod_{i=1}^t E_i }
	\right],
\end{align*}
where we used that $E_1,\dots,E_t$ are $\sigma(D_{t-1})$-measurable.
Now,
\begin{align*}
\E\left[ \exp\left(\sum_{i=1}^t \ell (\pi^i,\tau^i)\right)\mid D_{t-1}\right] 
=
\E\left[ \exp\left(\sum_{i=1}^{t-1} \ell (\pi^i,\tau^i)\right)\mid D_{t-1}\right]  E_t\,.
\end{align*}
Plugging this back into our previous expression,
\begin{align*}
u_t
&= 
 \E \left[
    \frac{ \E\left[ \exp(\sum_{i=1}^{t-1} \ell (\pi^i,\tau^i))\mid D_{t-1}\right]  \cancel{E_t} }
 					    { (\prod_{i=1}^{t-1} E_i)\cancel{E_t} }
	\right] 
= 
 \E \left[
 \E\left[
    \frac{  \exp(\sum_{i=1}^{t-1} \ell (\pi^i,\tau^i))  }
 					    { \prod_{i=1}^{t-1} E_i }
	\mid D_{t-1}\right]					    
	\right] \\	
&= 
 \E \left[
    \frac{  \exp(\sum_{i=1}^{t-1} \ell (\pi^i,\tau^i))  }
 					    { \prod_{i=1}^{t-1} E_i }
	\right] 	\tag{by the tower rule}\\
	&= u_{t-1} = \dots = u_0 = 1\,,
\end{align*}
where the second equality used that $E_1,\dots,E_{t-1}$ is $\sigma(D_{t-1})$-measurable.
\end{proof}

For any $\bar\theta\in\bar\Theta$, we define 
$$
\ell_{\bar\theta}(\pi,\tau) := 
\begin{cases}
	\frac12\log\left(\frac{\P^{\pi}_{\bar\theta}(\tau)}{\P^{\pi}_{{\theta}^\star}(\tau)}\right), &\P^{\pi}_{{\theta}^\star}(\tau)\neq 0,\\
0, &\mbox{otherwise,}
\end{cases}
$$
and 
$$ L_{\bar\theta}(D) = \sum_{i=1}^t \ell_{\bar\theta}(\pi^i,\tau^i), 
\quad L_{\bar\theta}(\tilde D) = \sum_{i=1}^t \ell_{\bar\theta}(\pi^i,\tilde \tau^i).$$

By Lemma \ref{lem:aux}, Chernoff's method
 and the union bound, 
with probability at least $1-\delta$, for all $\bar\theta\in\bar\Theta$ we have
\begin{align*}
	-\log \E_{\tilde D} [\exp(L_{\bar\theta}(\tilde D))\mid D] &< -L_{\bar\theta}(D) + \log(|\bar\Theta|/\delta) \\
	&\le -L_{\bar\theta}(D)  +   \cO(H(S^2A+SO))\log(TSAOH) + \log(1/\delta),
\end{align*}
Then, by the definition of $L_{\bar \theta}$ and by the inequality $-\log x \ge 1-x$,
$$
-\log \E_{ \tilde D} [\exp(L_{\bar\theta}(\tilde D))\mid D]
= - \sum_{i=1}^t \log\E_{\tau\sim \P^{\pi^i}_{{\theta}^\star}} \left[\sqrt{\frac{\P^{\pi^i}_{\bar{\theta} }(\tau)}{\P^{\pi^i}_{{\theta}^\star}(\tau)}}\right]
\ge \sum_{i=1}^t \left(1- \E_{\tau\sim \P^{\pi^i}_{{\theta}^\star}} \left[\sqrt{\frac{\P^{\pi^i}_{\bar{\theta} }(\tau)}{\P^{\pi^i}_{{\theta}^\star}(\tau)}}\right]\right).
$$
Recall that for any $\theta\in\Theta$ and any policy $\pi$, $\| \P_\theta^\pi -\P_{\bar\theta}^\pi\|_1 \le 1/T.$ As a result, by algebra,
\begin{align*}
\MoveEqLeft
\sum_{i=1}^t \left(1- \E_{\tau\sim \P^{\pi^i}_{{\theta}^\star}} \left[\sqrt{\frac{\P^{\pi^i}_{\bar{\theta} }(\tau)}{\P^{\pi^i}_{{\theta}^\star}(\tau)}}\right]\right)\\
&=
\sum_{i=1}^t \left(1-
\sum_\tau 
 \sqrt{\P^{\pi^i}_{{\theta}^\star}(\tau) \P^{\pi^i}_{\bar{\theta} }(\tau)}\right)\\
 &\ge  
 	\frac{1}{2}\sum_{i=1}^t \sum_\tau \left( \sqrt{\P^{\pi^i}_{\bar{\theta} }(\tau)} - \sqrt{\P^{\pi^i}_{{\theta}^\star}(\tau)}\right)^2 - \frac{1}{2} \\
 & \ge 
 	 \frac{1}{12}\sum_{i=1}^t \left[\sum_\tau \left( \sqrt{\P^{\pi^i}_{\bar{\theta} }(\tau)} - \sqrt{\P^{\pi^i}_{{\theta}^\star}(\tau)}\right)^2\right]
 \left[\sum_\tau \left( \sqrt{\P^{\pi^i}_{\bar{\theta} }(\tau)} + \sqrt{\P^{\pi^i}_{{\theta}^\star}(\tau)}\right)^2\right]
  - \frac{1}{2} \\
 & \ge 
 	 \frac{1}{12} \sum_{i=1}^t  \left(\sum_\tau | \P^{\pi^i}_{\bar{\theta} }(\tau) - \P^{\pi^i}_{{\theta}^\star}(\tau)|\right)^2- \frac{1}{2},
\end{align*}
where the last inequality follows from the Cauchy-Schwarz inequality. Putting all relations together, we have that with probability at least $1-\delta$, for all $\bar\theta\in\bar\Theta$
$$
 -L_{\bar\theta}(D)  +   \cO(H(S^2A+SO))\log(TSAOH) + \log(1/\delta) 
\ge \frac{1}{12} \sum_{i=1}^t  \left(\sum_\tau | \P^{\pi^i}_{\bar{\theta} }(\tau) - \P^{\pi^i}_{{\theta}^\star}(\tau)|\right)^2 - \frac{1}{2}.
$$
Finally, notice that for all $\theta\in\Theta$, we have
$$ -L_{\bar\theta}(D) \le -L_\theta(D).$$
Then, repeatedly using that $\| \P_\theta^\pi -\P_{\bar\theta}^\pi\|_1 \le 1/T$ and $\| \P^{\pi}_{\bar{\theta} } - \P^{\pi}_{{\theta}^\star}\|_1\le 2+1/T$ for any policy $\pi$, we get
$$
\sum_{i=1}^t  \left(\sum_\tau | \P^{\pi^i}_{\bar{\theta} }(\tau) - \P^{\pi^i}_{{\theta}^\star}(\tau)|\right)^2
\ge \sum_{i=1}^t  \left(\sum_\tau | \P^{\pi^i}_{{\theta} }(\tau) - \P^{\pi^i}_{{\theta}^\star}(\tau)|\right)^2 - 6.
$$
We conclude that with probability at least $1-\delta$, for all $\theta\in\Theta$
$$
 -L(\theta,D)  +   \cO(H(S^2A+SO))\log(SAOHT) + \log(1/\delta) 
\ge \frac{1}{12} \sum_{i=1}^t  \left(\sum_\tau | \P^{\pi^i}_{{\theta} }(\tau) - \P^{\pi^i}_{{\theta}^\star}(\tau)|\right)^2 - 1.
$$
Taking a union bound for all $t\in[T]$ completes the proof.

\section{$\ell_1$-norm eluder Dimension}
\label{app:eluder}

 In this section, we introduce the framework of $\ell_1$-norm eluder dimension, and present the corresponding pigeonhole-style regret guarantee. 
 All the proofs for this section are deferred to Appendix \ref{app:eluder-proof}.

 \subsection{Definitions and properties of $\ell_1$-norm eluder Dimension}
 To begin with, we define the $\epsilon$-independence relation between a point and a set of points with respect to a function class under the $\ell_1$-norm.

\begin{definition}[$\ell_1$-norm  $\epsilon$-independence]
\label{def:ind_points}
	Let $\cF$ be a function class defined on $\cX$, and let 
	\[z,x_1,x_2,\ldots,x_n\in\cX\,.\]
	We say 	$z$ is $\epsilon$-independent of $\{x_1,x_2,\ldots,x_n\}$ with respect to $\cF$ if there exists $f\in\cF$ such that $\sum_{i=1}^{n} | f(x_i)|\le \epsilon$, but $|f(z)| > \epsilon$. 
\end{definition}
When $\cF$ is clear from the context, we drop ``with respect to $\cF$'' for brevity.
We say 	$z$ is $\epsilon$-dependent on $\{x_1,x_2,\ldots,x_n\}$ if it is \emph{not} $\epsilon$-independent of $\{x_1,x_2,\ldots,x_n\}$.

\begin{definition}[$\ell_1$-norm $\epsilon$-eluder sequence]
We say that $\{x_i\}_{i=1}^n \subseteq \cX$ is an $\ell_1$-norm $\epsilon$-eluder sequence if
for all $i\in [n]$, $x_i$ is $\epsilon$-independent of $\{x_1,\dots,x_{i-1}\}$.
\end{definition}

\begin{definition}[$\ell_1$-norm $\epsilon$-eluder dimension]
\label{def:eluder}
Let $\cF$ be a function class defined on $\cX$.
	The  $\ell_1$-norm  $\epsilon$-eluder dimension, $\dim_{\rm E}(\cF,\epsilon)$, 
	is the length of the longest 
	$\ell_1$-norm $\epsilon'$-eluder sequence with some $\epsilon'\ge \epsilon$.
\end{definition}
If $\{x_i\}_{i=1}^n\subseteq \cX$ is an $\ell_1$-norm $\epsilon'$-eluder sequence with $n=\dim_{\rm E}(\cF,\epsilon)$ then we say that this sequence is witness to the $\ell_1$-norm eluder dimension of $\cF$.

The  difference between the $\ell_1$-norm eluder dimension defined here and the original eluder dimension \citep{russo2013eluder} is that in the definition of independence our $\ell_1$-version evaluates the $\ell_1$-norm of the function on the dataset instead of the $\ell_2$-norm as in the original definition. 
Therefore, we will refer to the original eluder dimension as $\ell_2$-norm eluder dimension throughout this paper, while we will drop the ``$\ell_1$-norm qualifier'' whenever it is clear from the context.

\begin{proposition}\label{prop:l1-l2-eluder}
The	$\ell_1$-norm $\epsilon$-eluder dimension is always upper bounded by the $\ell_2$-norm $\epsilon$-eluder dimension.
\end{proposition}

By combinining the upper bound for the $\ell_2$-norm eluder dimension of $d$-dimensional linear function classes with Proposition  \ref{prop:l1-l2-eluder}, we immediately obtain that the $\ell_1$-norm eluder dimension of any bounded $d$-dimensional linear function class is at most $\tilde{O}(d)$.
\begin{corollary}\label{cor:linear}
	The $\ell_1$-norm $\epsilon$-eluder dimension of 
	\[
	\Fcal=\{f_\theta :\ 
	f_\theta(x)=\langle x\,,\theta \rangle, x\in B_{R_2}^d(0)\,,
	\theta\in B_{R_1}^d(0)\}
	\]
	is at most $\tilde{\mathcal{O}}(d\log(1+R_1R_2/\epsilon))$.
\end{corollary}

\subsection{The pigeonhole principle for the $\ell_1$-norm eluder dimension }

Similar to $\ell_2$-norm eluder dimension, we can also prove a pigeonhole-style regret gurantee for the $\ell_1$ version.
This will play a key role in deriving the final regret bound from the MLE guarantee.

 \begin{proposition}\label{prop:de-regret}
Let  $\Phi$ be a set of real-valued functions sharing the domain $\cX$
and bounded by $C>0$.
	Suppose sequence $\{\phi_k\}_{k=1}^{K}\subset \Phi$ and $\{x_k\}_{k=1}^{K}\subset\cX$ satisfy that for all $k\in[K]$,
	\[
	\sum_{t=1}^{k-1}  |\phi_k(x_{ t})| \le \beta\,.
	\] 
	Then for all $k\in[K]$ and $\omega>0$,
	$$
	\sum_{t=1}^{k} | \phi_{t}(x_{ t})| \le (d+1)C + d\beta \log(C/\omega) + k\omega,
	$$
	where $d=\dim_{\rm E}(\Phi,\omega)$ is the  $\ell_1$-norm $\omega$-eluder dimension 
\end{proposition}

Finally, we instantiate Proposition \ref{prop:de-regret} on a linear function class  with Corollary \ref{cor:linear}, and obtain the following $\ell_1$-norm pigeonhole regret bound. 
\begin{proposition}\label{prop:linear-regret}
Suppose $\{w_{k,j}\}_{(k,j)\in[K]\times[m]},~\{x_{k,i}\}_{(k,i)\in[K]\times[n]}\subset\R^d$ satisfy
\begin{align}
\begin{cases}
\sum_{t=1}^{k-1} \sum_{i=1}^n \sum_{j=1}^m 
| w_{k,j}\trans x_{t,i}| \le \gamma_k  	\\
\sum_{i=1}^n \| x_{k,i} \|_2 \le R_x \\
\sum_{j=1}^m \| w_{k,j} \|_2 \le R_w\\
 \end{cases}
\quad  \mbox{ for all } k\in[K].
\label{eq:lrcond}
\end{align}
Then we have 
\begin{equation}
	\sum_{t=1}^{k} \sum_{i=1}^n\sum_{j=1}^m
| w_{t,j} \trans x_{t,i}| =\cO\bigg( d\left(R_w R_x + \max_{t\le k}\gamma_t \right) \log^2(Kn)  \bigg)
\quad  \mbox{ for all } k\in[K].
\end{equation}
\end{proposition}


\section{Proofs for $\ell_1$-norm eluder dimension}\label{app:eluder-proof}

In this section, we provide the proofs for the results in Appendix \ref{app:eluder}.

\subsection{Proof of Proposition \ref{prop:l1-l2-eluder}}
\begin{proof}[Proof of Proposition \ref{prop:l1-l2-eluder}]
	Let $x_1,\ldots,x_n$ be an $\ell_1$-norm  $\epsilon$-independent sequence.
	By definition, there exist  $f_1,\ldots,f_n\in\cF$ such that for all $k\in[n]$
	\begin{equation}
		\begin{cases}
			&\sum_{i=1}^{k-1} |f_k(x_i)| \le \epsilon,\\
			&|f_k(x_k)|\ge \epsilon.
		\end{cases}
	\end{equation}
Since $\ell_1$-norm is always an upper bound for the $\ell_2$-norm, we also have for all $k\in[n]$
\begin{equation}
		\begin{cases}
			&\sqrt{\sum_{i=1}^{k-1} |f_k(x_i)|^2} \le \epsilon,\\
			&|f_k(x_k)|\ge \epsilon.
		\end{cases}
	\end{equation}
Therefore, $x_1,\ldots,x_n$ is also an $\ell_2$-norm  $\epsilon$-independent sequence.
Hence, the $\ell_2$-eluder dimension of $\cF$ is at least as large as the $\ell_1$-eluder dimension of $\cF$
since any witness of the $\ell_1$-eluder dimension at some scale is a witness of the $\ell_2$-eluder dimension for the same scale.
\end{proof}

\subsection{Proof of Proposition \ref{prop:de-regret}}

The proofs in this subsection follow  the arguments developed for the analogous statements for the $\ell_2$-norm eluder dimension in Appendix C of \cite{russo2013eluder}.
We first prove a few auxiliary claims. 
\begin{claim}\label{claim:1}
Assume that the conditions of Proposition \ref{prop:de-regret} hold. 
Let $\epsilon>0$.
Then, if for some $k\in[K]$ we have $ | \phi_k(x_k)| > \epsilon $, then $x_k$ is $\epsilon$-dependent with respect to $\Phi$ on at most $\beta / \epsilon$ disjoint subsequences in $\{x_1,\dots,x_{k-1}\}$.
\end{claim}
\begin{proof}
Pick $k\in [K]$ as in the claim.
Let  $\{z_{1},\dots,z_{\ell}\}$ be a subsequence of $\{x_1,\dots,x_{k-1}\}$
such that $x_k$ is $\epsilon$-dependent on  $\{z_{1},\dots,z_{\ell}\}$ 
with respect to $\Phi$.
Since from
$\sum_{i=1}^{\ell} | \phi_k(z_i)|\le \epsilon$ it follows that $x_k$ is $\epsilon$-independent on  $\{z_{1},\dots,z_{\ell}\}$ 
with respect to $\Phi$, which we assumed not to hold, it follows that
$\sum_{i=1}^{\ell} | \phi_k(z_i)|> \epsilon$.
This implies that if $x_k$ is $\epsilon$-dependent on $L$ disjoint subsequences in $\{x_1,\dots,x_{k-1}\}$, we have
	 $$  \beta \geq \sum_{t=1}^{k-1} | \phi_k(x_t)| > L\epsilon,$$ which results in $L < {\beta}/{\epsilon}$.
 \end{proof}

\begin{claim}\label{claim:2}
Let $\Phi$ be an arbitrary class of real valued functions sharing the common domain $\cX$.
For any $\epsilon>0$, sequence $\{z_1,\dots,z_\kappa\} \subseteq \cX$, there exists $ j \in [\kappa]$ such that $z_j$ is $\epsilon$-dependent on at least $L = \lfloor (\kappa-1) / d_{\rm E} (\Phi,\epsilon)\rfloor $ disjoint subsequences in  $\{z_1,\dots,z_{j-1}\}$.
\end{claim}
\begin{proof}[Proof of Claim \ref{claim:2}]
	 Let $L$ be defined as in the statement of the claim. 
	 	 If $L=0$, there is nothing to be proven. Otherwise
	 we argue as follows. 	 

	 Initialize the sequences $B_1 = \{z_1\}, \dots,B_L$ $= \{z_L\}$, and let $j=L+1$.  Now, consider the following process: if at any point, $z_j$ is $\epsilon$-dependent on all of the $B_1,\dots,B_L$, then the claim is proven and we terminate the process. Otherwise, we pick an $i\in[L]$ such that $z_j$ is $\epsilon$-independent of $B_i$ and update $B_i:=B_i\cup\{z_j\}$. Then we increment $j$ by $1$ and continue this process if $j\le \kappa$.
	 
	 It remains to be proven that the process is terminated before $j$ gets to $\kappa+1$.
	 We prove this by contradiction:
	 Assume that $j$ gets to $\kappa+1$. Let $j$ be the index of a sequence from $\{B_i\}_{i\in [L]}$ that has the most elements (in case of ties, chose arbitrarily).
	 From $\kappa=\sum_{i=1}^L |B_i| \le L |B_j|$, 
	 $|B_j|\ge \kappa/L \ge \frac{\kappa}{\kappa-1} d_{\rm E}(\Phi,\epsilon) > d_{\rm E}(\Phi,\epsilon)$.
	 But $B_j$ is an $\epsilon$-eluder sequence by construction, which contradicts that
	 $d_{\rm E}(\Phi,\epsilon)$ is 
	 the length of longest $\epsilon'$-eluder sequences with $\epsilon'\ge \epsilon$.
\end{proof}

Equipped with Claims \ref{claim:1} and \ref{claim:2}, we can prove the following lemma, which bounds 
the frequency of large values in $\{ |\phi_1(x_1)|,\dots,|\phi_k(x_k)|\}$.
\begin{lemma}\label{lem:density}
	Under  the same condition of Proposition \ref{prop:de-regret}, for all $k\in[K]$,
$$
	\sum_{t=1}^{k} \mathbf{1}\big \{| \phi_t(x_t)| > \epsilon \big \} \leq \left(\frac{\beta}{\epsilon}+1\right) d_{\rm E} (\Phi,\epsilon)+1.
$$
\end{lemma}
  \begin{proof}
	 Fix $k \in [K]$ and let $\{z_1,\dots,z_\kappa\}$ be the subsequence of $\{x_{1},\dots,x_{k}\}$ consisting of elements for which $|\phi_t(x_t)| > \epsilon$.
	 By Claim \ref{claim:2}, we know there exists $j \in [\kappa]$ such that $z_j$ is $\epsilon$-dependent on  $\lfloor (\kappa-1) / d_{\rm E} (\Phi,\epsilon)\rfloor$ disjoint subsequences of $\{z_1,\dots, z_{j-1}\}$. 
	 By Claim \ref{claim:1}, $z_j$ is $\epsilon$-dependent on at most $\beta / \epsilon$ disjoint subsequences of $\{z_1,\dots,z_{j-1}\}$.
Therefore, we have
	 $$\lfloor (\kappa-1) / d_{\rm E} (\Phi,\epsilon)\rfloor  \leq \beta / \epsilon,$$ which implies
	 $$
	 	\kappa \leq \left(\frac{\beta}{\epsilon}+1\right)d_{\rm E}(\Phi,\epsilon)+1,
	 $$
	 completing the proof.
 \end{proof}

With this, we are ready to prove Proposition \ref{prop:de-regret}.
\begin{proof}[Proof of Proposition \ref{prop:de-regret}]
	Fix $k \in [K]$; let $d = d_{\rm E}(\Phi,\omega)$. Noting that $|\phi_t(x_t)| \leq C$, we have
	\begin{align*}
	\sum_{t=1}^{k} |\phi_t(x_t)| &= \sum_{t=1}^{k} \int_{0}^{C} \mathbf{1}\{ |\phi_t(x_t)|>y\} dy \\
&\le k \omega +  \sum_{t=1}^{k} \int_{\omega}^{C} \mathbf{1}\{ |\phi_t(x_t)|>y\}dy \\
&=  k\omega +   \int_{\omega}^{C}  \left(\sum_{t=1}^{k} \mathbf{1}\{ |\phi_t(x_t)|>y\}\right)dy \\
&\le  k\omega+ \int_{\omega}^{C}  \left( \left\{\frac{\beta}{y}+1\right\} d +1\right)dy  
\\
&\le  (d+1)C + d\beta \log(C/\omega) + k\omega,
\end{align*}
where the one but last inequality follows from Lemma \ref{lem:density} and the monotonicity of $\ell_1$-norm eluder dimension.
\end{proof}

\subsection{Proof of Proposition \ref{prop:linear-regret}}

\begin{proof}
	Given  $v\in[Kn]$,  denote $p_v = \lfloor \frac{v-1}{n} \rfloor+1$ and $q_v=v- n\lfloor \frac{v-1}{n} \rfloor$.
	Note that $1\le p_v\le K$ and $1\le q_v \le n$.
	We introduce the following two auxiliary sequences:
		\begin{equation}
			\begin{cases}
				&\zeta_{v,j} = w_{ p_v,j }\\
				& z_v = x_{p_v,q_v}
			\end{cases}
			 \quad (v,j)\in[Kn]\times[m]\,.
		\end{equation}
	By the definition of $(\zeta_{v,j})$,$(z_j)$ and  precondition \eqref{eq:lrcond},
		\begin{equation}
			\begin{aligned}
		\sum_{t=1}^{v-1} \sum_{j=1}^m  | \zeta_{v,j} \trans z_t|
			 =
			\sum_{k=1}^{p_v - 1} \sum_{i=1}^n \sum_{j=1}^m
		| w_{p_v,j} \trans x_{k,i}|  +
		\sum_{i=1}^{q_v-1}\sum_{j=1}^m
		| w_{p_v,j} \trans x_{p_v,i}| \le \gamma_{p_v} +  R_w R_x. 
			\end{aligned}
		\end{equation}
	Notice that for all $k\in[K]$,
		$$\sum_{v=1}^{kn} \sum_{j=1}^{m}   | \zeta_{v,j} \trans z_{v}|=
		\sum_{t=1}^{k} \sum_{i=1}^n \sum_{j=1}^m
		| w_{t,j} \trans x_{t,i}|.$$
So in order to prove 	Proposition \ref{prop:linear-regret}, it suffices to upper bound the value of the following  optimization problem:
\begin{equation}\label{opt:1}
\begin{aligned}
    &\max_{\zeta,z} \sum_{v=1}^{kn}
     \sum_{j=1}^m 
     \left|\langle\zeta_{v,j}, z_v \rangle\right|\\
     \text{s.t. for any $v$}: &  \begin{cases} \sum_{t=1}^{v-1} \sum_{j=1}^m  | \zeta_{v,j} \trans z_t|\le \gamma_{p_v} +  R_w R_x,\\
		\sum_{j=1}^m \| \zeta_{v,j} \|_2 \le R_w, \\
		 \| z_{v} \|_2 \le R_x.
     \end{cases}
 \end{aligned}
\end{equation}
Now, we make the observation that the above optimization problem has the same optimal value as the following one (we assume all vectors in this section have the same dimension):
\begin{equation}\label{opt:2}
	\begin{aligned}
		&\max_{\zeta,z} \sum_{v=1}^{kn} 
		 \left|\langle\zeta_{v}, z_v \rangle\right|\\
		 \text{s.t. for any $v$}: &  \begin{cases} \sum_{t=1}^{v-1}   | \zeta_{v} \trans z_t|\le \gamma_{p_v} +  R_w R_x,\\
			 \| \zeta_{v} \|_2 \le R_w, \\
			 \| z_{v} \|_2 \le R_x.
		 \end{cases}
	 \end{aligned}
	\end{equation}
	It is direct to see the optimal value of Problem \eqref{opt:2} is always no larger than that of Problem \eqref{opt:1}. For the other direction, suppose $\{\zeta_{v,j}^\star\}_{(v,j)\in[Kn]\times[m]}$ and $\{z^\star_v\}_{v\in[Kn]}$ are   optimal solution to Problem \eqref{opt:1}. Consider 
    $$
\zeta_{v}^\star = \sum_{j=1}^m \zeta_{v,j}^\star \times \text{sign}({\zeta_{v,j}^\star}\trans z_v^\star)\quad  \text{ for } k\in[K]. $$
One can easily verify $\{\phi_k^\star\}_{k\in[K]}$ and $\{z^\star_v\}_{v\in[Kn]}$ are  feasible solution to Problem \eqref{opt:1} by triangle inequality. 
Moreover, its objective value is 
$$
\sum_{v=1}^{kn} 
    \left| \sum_{j=1}^m \left(\zeta_{v,j}^\star \times \text{sign}({\zeta_{v,j}^\star}\trans z_v^\star)\right) \trans z_{v}^\star\right|
    = 
    \sum_{v=1}^{kn}   \sum_{j=1}^m \left|{\zeta_{v,j}^\star}\trans z_v^\star\right|,
$$
which is equal to the optimal value of Problem \eqref{opt:1}.
Therefore, it suffices to upper bound the optimal value of Problem \eqref{opt:2}.

Finally, by applying  Proposition \ref{prop:de-regret} to $d$-dimensional linear function class, we can  upper bound the optimal value of Problem \eqref{opt:2} by
$$
\min_{\omega \le C} \left\{ d_E(\omega) C + d_E(\omega) \lambda \log(C/\omega) + kn\omega \right\}
$$
where 
$$
\begin{cases}
d_E(\omega) = d\log(e+R_w R_x/\omega)\\
C =  R_w R_x\\
\lambda = \gamma_{p_v} +  R_w R_x.
\end{cases}
$$
We conclude the proof by choosing $\omega = R_w R_x/(kn)$.
\end{proof}

\section{Proofs for Undercomplete POMDPs}
\label{app:proof-under}

In this section, we prove  Theorem \ref{thm:under} with a specific  polynomial dependency as stated in the following theorem.
\begin{theorem}[Regret of \omle] \label{thm:under-rate}
	There exists an absolute constant $c>0$ such that for any $\delta\in(0,1]$ and $S,A,O,H,K\in\N$,  if we choose $\beta = c\left(H(S^2A+SO)\log(SAOHK)+\log(K/\delta)\right)$ in Algorithm \ref{alg:under}, then,
	for any POMDP with $S$ states, $A$ actions, $O$ observations and horizon $H$ and
	satisfying Assumption \ref{asp:under},
	  with probability at least $1-\delta$,%
	  $$   {\rm Regret}(k)\le \tilde{\cO}\left( \frac{S^{2}AO}{\alpha^2} \sqrt{k(S^2A+SO)} \times \poly(H)\right)  \qquad \text{ for all }k\in[K].$$
	\end{theorem}

To begin with, we recall (or introduce) some notations that will be needed in the proof.

\paragraph{Notation} Denote by $\tau_h=(o_1,a_1,\ldots,o_h,a_h)$  a trajectory up to step $h$. 
(That is, for a natural number $h$, $\tau_h\in (\fO \times \fA)^h$, by convention.)
Given a policy $\pi$ and a trajectory $\tau_h$, define $\pi(\tau_h) = \prod_{h'=1}^h \pi(a_{h'}\mid \tau_{h'-1},o_{h'})$. This is the part of the probability of  $\tau_h$ that can be attributed to the policy $\pi$.
We use $\theta^\star=(\T,\O,\mu_1)$ to denote the parameters of the  groundtruth POMDP that the learner interacts with, and denote by $\B$ the observable operators corresponding to $\theta^\star$ and by 
 $\b_0$ the initial observation distribution. Formally,
\begin{equation}\label{eq:op-under}
	\begin{cases}
&  \b_0 = \O_1 \mu_1\in\R^O,\\
& \B_h(o,a) = \O_{h+1}\T_{h,a}  \diag(\O_h(o\mid \cdot))\O_h^\dagger\in\R^{O\times O},
	\end{cases}
\end{equation}
where $\O_h(o\mid \cdot)$ is the $o^{\rm th}$ row of the observation matrix $\O_h$ and $\diag(\O_h(o\mid \cdot))$ is a diagonal matrix with diagonals equal to $\O_h(o\mid \cdot)$. 
We use $\theta^k$ to denote the optimistic estimate of the POMDP model  in the $k^{\rm th}$ episode of Algorithm \ref{alg:under}, and denote by $\pi^k$ the optimal policy of  model $\theta^k$, i.e., 
 $(\theta^k,\pi^k) = \argmax_{\theta\in\cB^{k-1},\pi}V_1^\pi( \theta)$. 
 Finally, we will use $\B^k$ and $\b_0^k$ to denote the operators corresponding to $\theta^k$.

  For a matrix $\A$, $p\ge 1$, let $\|\A\|_{p} = \sup_{x:\norm{x}_p\le 1} \norm{\A x}_p$.
We start with a lemma that bounds $\|\B_h(o,a)\|_{1}$ as a function of $\alpha$:
\begin{lemma}\label{lem:11bound}
Under Assumption~\ref{asp:under} it holds that
for any $(o,a,h,k)\in \fO\times \fA\times[H-1] \times \N$, 
$$\|\B_h(o,a)\|_{1},\|\B_h^k(o,a)\|_{1}  \le \sqrt{S}/\alpha.$$
\end{lemma} 
\begin{proof}
By definition,
\begin{align*}
	\|\B_h(o,a)\|_{1} &= \|\O_{h+1}\T_{h,a}  \diag(\O_h(o\mid \cdot))\O_h^\dagger\|_{1}\\
	& \le  \|\O_{h+1}\|_{1} \times \|\T_{h,a}\|_{1} \times \|\diag(\O_h(o\mid \cdot))\|_{1}\times\|\O_h^\dagger\|_{1}\\
	& \le \|\O_h^\dagger\|_{1} 
	=\sup_{x:\|x\|_1\le 1}\|\O_h^\dagger x \|_{1} \le\sqrt{S}\sup_{x:\|x\|_2\le 1}\|\O_h^\dagger x \|_{2} \le \sqrt{S}/\alpha,
\end{align*}
where the last inequality uses Assumption \ref{asp:under}. The same arguments also give $\|\B_h^k(o,a)\|_{1}  \le \sqrt{S}/\alpha$.
\end{proof}
 
\paragraph{Computing probabilities using the operators}
As it is well known \citep[e.g.,][]{jin2020sample}, given a set of POMDP parameters $\theta$, its corresponding operators satisfy the following relations for any policy $\pi$ and trajectory $\tau_h=(o_1,a_1,\ldots,o_h,a_h)$:
 \begin{equation}\label{eq:prob-op}
	\begin{cases}
	 \P_\theta^\pi(\tau_h) = \pi(\tau_h)\cdot
	 \left(\e_{o_h}\trans\B_{h-1}(o_{h-1},a_{h-1};\theta)\cdots\B_1(o_1,a_1;\theta)\b_0(\theta)\right)\in\R,\\
	 \P_\theta^\pi(\tau_h,o_{h+1}) = \pi(\tau_h)\cdot
	 \left(\e_{o_{h+1}}\trans\B_h(o_{h},a_{h};\theta)\cdots\B_1(o_1,a_1;\theta)\b_0(\theta)\right)\in\R,\\
	 \P_\theta^\pi(o_1)= \e_{o_1}\trans\b_0(\theta)\in\R^O,
	\end{cases}
 \end{equation}
 where $\P_\theta^\pi(\tau_h)$ denotes the probability of observing $\tau_h$ under policy $\pi$ in the POMDP model defined by $\theta$, and $(\{\B_h(o,a;\theta)\}_{(o,a,h)\in\fO\times\fA\times[H-1]},\b_0(\theta))$ denote the operators  corresponding to $\theta$ (Equation~\eqref{eq:ps-3} in the main body is the first case in Equation~\eqref{eq:prob-op}.)

 \paragraph{Proof roadmap} The proof consists of three main steps. In the first step, we upper bound the distribution estimation error  by the   operator estimation error. The proof of this follows arguments similar to those used by \cite{jin2020sample}.
 Then we leverage the MLE guarantee developed in Appendix \ref{app:mle} to derive certain constraints satisfied by the operator estimation error. Finally, we use the $\ell_1$-norm  eluder dimension framework developed in Appendix \ref{app:eluder} to bridge the results in the first two steps and obtain the desired regret guarantee.

\subsection{Step $1$: bound the regret by the error of operator estimates}
We first present a lemma that upper bounds the cumulative suboptimality of $\pi_1,\ldots,\pi_k$ by the cumulative density estimation error. 
\begin{lemma}\label{lem:opt2dist}
In Algorithm \ref{alg:under} and \ref{alg:over}, if we choose $\beta$ according to Theorem \ref{thm:under} and  \ref{thm:over} respectively, then with probability at least $1-\delta$, 
\begin{equation}\label{eq:optimisim-under}
	\sum_{t=1}^k V^{\pi^t}( {\theta^t}) - V^{\pi^t}( \theta^\star)
	   \le  H \sum_{t=1}^k \sum_{\tau_H   } | \P^{\pi^t}_{{\theta^t}}(\tau_H) - \P^{\pi^t}_{\theta^\star}(\tau_H)|.
   \end{equation}
\end{lemma}
\begin{proof}
By the choice of $\beta$ and Proposition \ref{prop:mle-optimism}, we have $\theta^\star\in\cB^t$ for all $t\in[K]$ with probability at least $1-\delta$.
In what follows, we assume that the event $\theta^\star\in\cap_{t\in [K]}\cB^t$ holds.
On this event, by the optimism of $\theta^t$ and $\pi^t$ for $t\in [k]$,
\begin{equation}
\sum_{t=1}^k V^{\star}( {\theta^\star}) - V^{\pi^t}( \theta^\star)
\le 
\sum_{t=1}^k V^{\pi^t}( {\theta^t}) - V^{\pi^t}( \theta^\star).
\end{equation}
Because the cumulative reward of each trajectory is bounded by $H$, we conclude
\begin{equation}
 \sum_{t=1}^k V^{\pi^t}( {\theta^t}) - V^{\pi^t}( \theta^\star)
	\le  H \sum_{t=1}^k \sum_{\tau_H   } | \P^{\pi^t}_{{\theta^t}}(\tau_H) - \P^{\pi^t}_{\theta^\star}(\tau_H)|.
\end{equation}
\end{proof}
As a result, to prove Theorem \ref{thm:under},  it suffices to upper bound the RHS of Equation \eqref{eq:optimisim-under}.

To begin with, we represent the probability of observing $\tau_H$ by  the product of operators using  Equation \eqref{eq:prob-op}, which gives 
\begin{align*}
	\MoveEqLeft \sum_{t=1}^k \sum_{\tau_H   } \left| \P^{\pi^t}_{{\theta^t}}(\tau_H) - \P^{\pi^t}_{\theta^\star}(\tau_H)\right|\\
	&=   \sum_{t=1}^k \sum_{\tau_H   } \left| \e_{o_{H}}\trans \left(\prod_{h=1}^{H-1}\B^t_h(o_{h},a_{h})\right) \b^t_0 - 
	\e_{o_H}\trans\left(\prod_{h=1}^{H-1}\B_h(o_{h},a_{h})\right)\b_0\right| \times \pi^t(\tau_H)\\
	&= \sum_{t=1}^k \sum_{\tau_{H-1}   } \left\|  \left(\prod_{h=1}^{H-1}\B^t_h(o_{h},a_{h})\right) \b^t_0 - 
	\left(\prod_{h=1}^{H-1}\B_h(o_{h},a_{h})\right)\b_0\right\|_1 \times \pi^t(\tau_{H-1}).
\end{align*}
	By using the lemma below, 
	we can  further control the difference between two  products of operators by the difference between each pair of operators. The proof of a more general version of this lemma is given in Section~\ref{sec:auxlemmas}.
\begin{lemma}\label{lem:prod-triangle-under}
For any $k\in\N$, $h\in[H-1]$ and policy $\pi$
\begin{align*}
	\MoveEqLeft
\sum_{\tau_{h}  } \left\| \B^k_h(o_{  h},a_{  h})\cdots \B^k_1(o_1,a_1)\b^k_0 -  \B_h(o_h,a_h)\cdots \B_1(o_1,a_1)\b_0\right\|_1	\times \pi(\tau_{h})\\
	&\le  \frac{\sqrt{S}}{\alpha}  \left(\sum_{j=1}^{h} \sum_{\tau_{j}   } 
	\left\| (\B^k_j(o_j,a_j) - \B_j(o_j,a_j))\b(\tau_{j-1})\right\|_1 \times \pi(\tau_{j}) + \|\b^k_0-\b_0\|_1\right),
	\end{align*}
	where for a trajectory $\tau_h=(o_1,a_1,\ldots,o_h,a_h)$, $
\b(\tau_h) = 
\left(\prod_{h'=1}^h \B_{h'}(o_{h'},a_{h'}) \right)\b_0$.
\end{lemma}
As noted in the main text, this lemma abuses notations in a few ways: In the innermost sum over the observation-action trajectories $\tau_j$ of length $j$, $\tau_{j-1}$ refers to the prefix of $\tau_j$ where the last observation-action is dropped. Also, in this sum, $(o_j,a_j)$ refer to the last observation-action pair of $\tau_j$.

Putting things together, we get
\begin{align*}
	\MoveEqLeft
	\sum_{t=1}^k \sum_{\tau_H   } | \P^{\pi^t}_{{\theta^t}}(\tau_H) - \P^{\pi^t}_{\theta^\star}(\tau_H)|\\
	&\le 
	\sum_{t=1}^k \sum_{\tau_{H-1}   } \left\|  \left(\prod_{h=1}^{H-1}\B^t_h(o_{h},a_{h})\right) \b^t_0 - 
	\left(\prod_{h=1}^{H-1}\B_h(o_{h},a_{h})\right)\b_0\right\|_1 \times \pi^t(\tau_{H-1})\\
	&\le  \frac{\sqrt{S}}{\alpha}  \left(\sum_{t=1}^k  \sum_{h=1}^{H-1} \sum_{\tau_{h}   } 
	\left\| \left(\B_h(o_h,a_h) - \B^t_h(o_h,a_h)\right)\b(\tau_{h-1})\right\|_1 \times \pi^t(\tau_{h}) + \|\b_0-\b^t_0\|_1\right). \numberthis \label{eq:tvbound1}
\end{align*}

\subsection{Step $2$: derive  constraints for the operator estimates from OMLE}
\label{subsec:step2}

By the construction of $\cB^k$,
$ \sum_{i=1}^{k-1} \log\left(\frac{\P^{\pi^i}_{{\theta}^\star}(\tau ^i)}{\P^{\pi^i}_{{\theta^k}}(\tau ^i)}\right)\le \beta$.
Therefore, by Proposition \ref{prop:mle-valid} and the choice of $\beta$, 
we have with probability at least $1-\delta$,
\begin{equation*}
	\text{ for all $k\in[K]$: } \quad \sum_{t=1}^{k-1}\left(\sum_{\tau_H}  \left|\P^{\pi^t}_{\theta^k}(\tau_H) - \P^{\pi^t}_{\theta^\star}(\tau_H)\right|\right)^2 = \Ocal( \beta).
\end{equation*}
In the rest of the proof, we assume that the event above is true.

By the Cauchy-Schwarz inequality, 
\begin{equation*}
	\sum_{t=1}^{k-1}\sum_{\tau_H}  \left|\P^{\pi^t}_{\theta^k}(\tau_H) - \P^{\pi^t}_{\theta^\star}(\tau_H)\right| = \Ocal( \sqrt{\beta k}), \  \text{ for all } k\in[K].
\end{equation*}
Since marginalizing two distributions will not increase their distance, from the previous inequality
we also have that for all $(k,h)\in[K]\times[H-1]$,
\begin{equation*}
\begin{cases}
	\sum_{t=1}^{k-1} \sum_{\tau_h,o_{h+1}} \left| \P^{\pi^t}_{\theta^k}(\tau_h,o_{h+1}) - \P^{\pi^t}_{\theta^\star}(\tau_h,o_{h+1})\right| = \Ocal( \sqrt{\beta k}),\\
	\sum_{t=1}^{k-1}\sum_{o_1}\left| \P^{\pi^t}_{\theta^k}(o_{1}) - \P^{\pi^t}_{\theta^\star}(o_{1})\right| = \Ocal( \sqrt{\beta k}).
\end{cases}
\end{equation*}
By equation \eqref{eq:prob-op}, we can replace the probability with the product of operators and obtain  for all $(k,h)\in[K]\times[H-1]$
\begin{equation}\label{eq:mle-1}
	\begin{cases}
	\sum_{t=1}^{k-1}  \sum_{\tau_h } \pi^t(\tau_h) \times \left \| \b^k(\tau_h) - \b(\tau_h)\right\|_1	= \Ocal( \sqrt{\beta k}), \\
\left \| \b_0^k - \b_0 \right \|_1 =  \Ocal( \sqrt{\beta/k}),
	\end{cases}
\end{equation}
where 
we used that the initial distributions over the observations are independent of the policies used and where 
$
\b(\tau_h) = 
 \left(\prod_{h'=1}^h \B_{h'}(o_{h'},a_{h'}) \right)\b_0
$ and $
\b^k(\tau_h) = 
 \left(\prod_{h'=1}^h \B^k_{h'}(o_{h'},a_{h'}) \right)\b_0^k.
$

The second part of Equation \eqref{eq:mle-1} shows that the parameter estimates are improving in terms of how well the initial observation distribution is estimated (which is unsurprising).
Our goal now is to use the first part of this equation to show that the estimated observable operators are also getting closer to their groundtruth counterparts.
To show this, first notice that 
by the triangle inequality, for any $(k,h)\in[K]\times[H-1]$,
\begin{align*}
 \MoveEqLeft  \sum_{t=1}^{k-1}  \sum_{ \tau_h} \pi^t(\tau_h) \times 
	\left\| \left(\B^k_h(o_h,a_h) - \B_h(o_h,a_h)\right)\b(\tau_{h-1})\right\|_1 \\
	\le &  \sum_{t=1}^{k-1}   \sum_{ \tau_h} \pi^t(\tau_h) \times 
	\left\| \B^k_h(o_h,a_h) \left(\b^k(\tau_{h-1})-\b(\tau_{h-1})\right)\right\|_1  \\
	&+ 
	 \sum_{t=1}^{k-1}  \sum_{ \tau_h} \pi^t(\tau_h) \times
	\left\| \B^k_h(o_h,a_h)\b^k(\tau_{h-1}) - \B_h(o_h,a_h)\b(\tau_{h-1})\right\|_1.
	\end{align*}
	In the last display, 
	the second term is at most $\cO(\sqrt{k\beta})$ by \eqref{eq:mle-1}. For the first term, we can apply Lemma \ref{lem:OP3} and obtain
	\begin{align*}
	 \sum_{t=1}^{k-1}  \sum_{ \tau_h} \pi^t(\tau_h) \times 
	\left\| \B^k_h(o_h,a_h) \left(\b^k(\tau_{h-1})-\b(\tau_{h-1})\right)\right\|_1 =  \cO\left(\frac{\sqrt{S}}{\alpha} \sqrt{k\beta}\right).
	\end{align*}
Putting things together, we conclude that  with probability at least $1-\delta$: 
	\begin{equation}\label{eq:mle-2}
		\sum_{t=1}^{k-1}  \sum_{ \tau_h} \pi^t(\tau_h) \times 
	\left\| \left(\B_h^k(o_h,a_h) - \B_h(o_h,a_h)\right)\b(\tau_{h-1})\right\|_1= \cO\left( \frac{\sqrt{S}}{\alpha} \sqrt{k\beta}\right)
	\end{equation}
	for all $(k,h)\in[K]\times[H-1]$.

\subsection{Step $3$: bridge Step $1$ and $2$ via $\ell_1$-norm eluder dimension}
\label{subsec:step3}

By Step $1$ (in particular, by Equation~\eqref{eq:tvbound1}), 
to control the total estimation error of the distributions over trajectories, 
it suffices to upper bound the following quantity for all $k\in[K]$:
$$
\textbf{Target: }\sum_{t=1}^k \left(\sum_{h=1}^{  H-1} \sum_{\tau_{h}   } 
\left\| \left(\B^t_h(o_h,a_h) - \B_h(o_h,a_h)\right)\b(\tau_{h-1})\right\|_1 \times \pi^t(\tau_{h}) + \left\|\b^t_0-\b_0\right\|_1\right).
$$
By Step $2$, we have that with probability at least $1-\delta$: for all  $(k,h)\in[K]\times[H-1]$
\begin{equation*}
\textbf{Condition: }
\begin{cases}
	&\left \| \b_0^k - \b_0 \right \|_1   = \Ocal( \sqrt{\beta/k}),\\
&\sum_{t=1}^{k-1}  \sum_{ \tau_h} 
	\left\| \left(\B^k_h(o_h,a_h) - \B_h(o_h,a_h)\right)\b(\tau_{h-1})\right\|_1 \times \pi^t(\tau_h) = \cO\left( \frac{\sqrt{S}}{\alpha} \sqrt{k\beta}\right).
\end{cases}
\end{equation*}

From now on, we assume that the two conditions of the last display hold and will show that this suffices to bound the ``target'' above.
By using the first inequality in the condition, the summation of the second term in the target is upper bounded by $\cO(\sum_{t=1}^k \sqrt{\beta/t}) = \cO(\sqrt{k\beta})$.  It remains to control the summation of the first term in the target.

Now let us consider an \emph{arbitrary fixed} triple $(o,a,h)\in\fO \times\fA \times[  H-1]$. Denote by $\X_{l}$ the $l^{\rm th}$ row of matrix $\X$. 
The condition implies that 
\begin{align}\label{eq:step3-1}
	\sum_{t=1}^{k-1}  \sum_{ \tau_h:\ (o_h,a_h)=(o,a)} \sum_{l\in[O]}
	\left| \left[\left(\B^k_h(o,a) - \B_h(o,a)\right)\O_h\right]_{l} \O_{h}^\dagger\b(\tau_{h-1})\right| \times \pi^t(\tau_h)
	 = \cO\left( \frac{\sqrt{S}}{\alpha} \sqrt{k\beta}\right).
\end{align}
To further simplify the notations, for $(t,l)\in [K]\times[O]$, let $w_{t,l} :=\left[\left(\B^t_h(o,a) - \B_h(o,a)\right)\O_h\right]_{l}$  and 
denote
 the $n:=(OA)^{h-1}$ not necessarily distinct elements in the 
sequence
$\{ \O_h^\dagger \b(\tau_{h-1})\times \pi^t(\tau_{h}): \ \tau_h=(o_1,a_1,\dots,o_{h-1},a_{h-1},o,a)\in (\fO\times\fA)^{h-1}\times(o,a) \}$
by $x_{t,1},\dots,x_{t,n}$. Using the newly defined notations, Equation \eqref{eq:step3-1} is equivalent to
\begin{equation}
	\sum_{t=1}^{k-1} \sum_{l=1}^{O}\sum_{i=1}^n
| w_{k,l}\trans x_{t,i}| = \cO\left( \frac{\sqrt{S}}{\alpha} \sqrt{k\beta}\right).
\end{equation}
We have three observations about the $x,w$ sequences:
\begin{itemize}
\item The vectors  $\{x_{t,i}\}_{i=1}^n$ satisfy $
\sum_{i=1}^n
\| x_{t,i}\|_1 \le 1$ for all $t$ because
\begin{align*}
\sum_{ \tau_h:\ (o_h,a_h)=(o,a)}  \|\O_h^\dagger \b(\tau_{h-1})\pi^t(\tau_h)\|_1
&\le \sum_{ \tau_h:\ (o_h,a_h)=(o,a)}  \|\O_h^\dagger\b(\tau_{h-1})\pi^t(\tau_{h-1})\|_1\\
&= \sum_{ \tau_{h-1}}  \|\O_h^\dagger\b(\tau_{h-1})\pi^t(\tau_{h-1})\|_1=1,
\end{align*}
where the final equality follows from $\|\O_h^\dagger\b(\tau_{h-1})\pi(\tau_{h-1})\|_1 = \P_{\theta^\star}^\pi(\tau_{h-1})$.  
\item The vectors  $\{w_{t,l}\}_{l=1}^O$ satisfy $
\sum_{l=1}^O
\| w_{t,l}\|_1 \le 1$ for all $t$. By definition and Lemma \ref{lem:11bound},   
$$
\sum_{l=1}^O
\| w_{t,l}\|_1 =  \| (\B_h(o,a)-\B_h^t(o,a)) \O_{h}\|_1\le S(\|\B_h(o,a)\|_1+\|\B^t_h(o,a)\|_1)\le \frac{2S^{1.5}}{\alpha}.
$$	
\end{itemize}
As a result, we can invoke Proposition \ref{prop:linear-regret} and obtain
\begin{equation}
	\sum_{t=1}^{k} \sum_{l=1}^O \sum_{i=1}^n
| w_{t,l}\trans x_{t,i}| = \tilde{\cO}\left( \frac{S^{1.5}H^2}{\alpha} \sqrt{k\beta}\right),
\ \text{ for all } k\in[K]. \label{eq:lrcor}
\end{equation}
Translating the above inequality back to the operator language gives that for an \emph{arbitrary fixed} tuple $(o,a,h)\in\fO\times\fA\times[  H-1]$
\begin{align*}
	\sum_{t=1}^{k}  \sum_{ \tau_h:\ (o_h,a_h)=(o,a)} 
	\left| \left(\B^t_h(o,a) - \B_h(o,a)\right)_{l}\b(\tau_{h-1})\right| \times \pi^t(\tau_h) = \tilde{\cO}\left( \frac{S^{1.5}H^2}{\alpha} \sqrt{k\beta}\right),
	\ \text{ for all } k\in[K]. 
\end{align*}
Therefore, we conclude that with probability at least $1-\delta$, for all $k\in[K]$,
$$
\sum_{t=1}^k\left(\sum_{h=1}^{  H-1} \sum_{\tau_{h}   } 
\left\| \left(\B^t_h(o_h,a_h) - \B_h(o_h,a_h)\right)\b(\tau_{h-1})\right\|_1 \times \pi^t(\tau_{h}) + \left\|\b^t_0-\b_0\right\|_1\right)= \tilde{\cO}\left( \frac{S^{1.5}OAH^3}{\alpha} \sqrt{k\beta}\right).
$$
Combining the above equality with Step $1$ gives:
$$   {\rm Regret}(k)\le \tilde{\cO}\left( \frac{S^{2}AO}{\alpha^2} \sqrt{k(S^2A+SO)} \times \poly(H)\right)  \qquad \text{ for all }k\in[K].$$
\begin{remark}
	We remark that the regret above  can be improved by a factor of $O$ if we adopt a more active  strategy for exploration in \omle. The modified algorithm is a special case of Algorithm \ref{alg:over} with $m=1$, which we decribe in details in Section \ref{sec:over}.
\end{remark}

\section{Proofs for Overcomplete POMDPs}
\label{app:proof-over}
In this section, we prove  Theorem \ref{thm:over} with a specific  polynomial dependency as stated in the following theorem.
\begin{theorem}[Total suboptimality of multi-step \omle]\label{thm:over-rate}
	There exists an absolute constant $c>0$ such that for any $\delta\in(0,1]$ and $S,A,O,K,H\in\N$, if we choose parameter $\beta$ in Algorithm \ref{alg:over} as $\beta = c\left(H(S^2A+SO)\log(SAOH)+\log(KH/\delta)\right)$, then, for any POMDP 
	with $S$ states, $A$ actions, $O$ observations and horizon $H$ and
	satisfying Assumption \ref{asp:over},
		with probability at least $1-\delta$, 
		$$   \sum_{t=1}^k \left( V^{\star}-V^{\pi^t} \right) \le \tilde{\Ocal}\left( \frac{S^2 A^{3m-2} }{\alpha^2}\sqrt{k(S^2A+SO)}\times\poly(H)\right) \qquad \text{ for all }k\in[K].$$
	\end{theorem}

The proof follows a similar three-step recipe to the undercomplete case: first we show that to control the regret it suffices to control the operator estimation error, then we derive some  constraints for the operator estimation from the MLE guarantee, and finally we utilize the constraints to upper bound the estimation error with the help of the $\ell_1$-norm eluder dimension framework. 
The main difference to the previous proof is only that the data generated for building the confidence sets is based on adding exploration to the policies $\pi^1,\dots,\pi^K$.

The reader may find it useful to first read Appendix \ref{app:proof-under} before reading this section as we will heavily reuse the notations and techniques developed there.  

\paragraph{Notation} 
As usual, we start by introducing (or recalling) the definitions of several notations that will be used in the proofs. 
Recall we 
define the $m$-step emission-action matrices $\{\M_h\in\R^{(A^{m-1}O^m)\times S}\}_{h\in[H-m+1]}$  as follows: 
For an observation sequence $\o$ of length $m$, initial state $s$ and action sequence $\a$ of length $m-1$,
we let $[\M_h]_{(\a,\o),s}$ be the probability of receiving  $\o$ provided that the action sequence $\a$ is used from state $s$ and step $h$: for all $(\a,\o)\in \fA^{m-1}\times \fO^m$ and $s\in\fS$
\begin{equation}
	[\M_h]_{(\a,\o),s}= \P(o_{h:h+m-1}=\o \mid 
	s_h=s,a_{h:h+m-2} =\a ).
\end{equation}

With slight abuse of notations, we reuse $\B$ and $\b$ to denote the operators for the general setting. Specifically, the operators corresponding to the groundtruth model  $\theta^\star$ is defined as:
\begin{equation}\label{eq:op-over}
\begin{cases}
&  \b_0 = \M_1 \mu_1\in\R^{A^{m-1}O^m},\\
& \B_h(o,a) = \M_{h+1} \T_{h,a} \diag(\O_h(o\mid \cdot)) \M_h^\dagger\in\R^{(A^{m-1}O^m)\times (A^{m-1}O^m)},
\end{cases}
\end{equation}
where the only difference from the undercomplete version (equation \eqref{eq:op-under}) is that we replace the single-step emission matrices  with the $m$-step emission-action matrices.

We use $\theta^k$ to denote the optimistic estimate of the POMDP model in the $k^{\rm th}$ iteration of Algorithm \ref{alg:over}, and denote by $\pi^k$ the optimal policy for model  $\theta^k$, i.e., 
$(\theta^k,\pi^k) = \argmax_{\theta\in\cB^{k-1},\pi}V_1^\pi(s_1;\theta)$. We  use $\B^k$ and $\b_0^k$ to denote the operators corresponding to $\theta^k$.

To simplicity the notation, throughout this section, we will use $\u\in(\fO\times\fA)^{m-1}\times \fO$ to denote a length-$(2m-1)$ observation-action sequence (the observations are interleaved with actions in $\u$). We use $\u_a\in\fA^{m-1}$ to refer to the action part of $\u$ and $\u_o\in\fO^{m}$ for the observation part.
We further define $\U:=(\fO\times \fA)^{m-1}\times \fO$ to be the collections of all the $(2m-1)$-length action-observation sequences.

\paragraph{Computing probabilities using the operators}
Similar to the undercomplete case, several useful relations hold for the newly defined operators.
In particular, for any policy $\pi$, trajectory $\tau_h$ (with $h\le H-m$ ) and observation-action sequence $\u\in \U$, it holds that 

\begin{equation}\label{eq:prob-op-over}
\begin{cases}
\P_\theta^\pi(\tau_h) \cdot \P_\theta(o_{h+1:h+m}=\u_o\mid \tau_h,a_{h+1:h+m-1}=\u_a) \\
\qquad \qquad = \pi(\tau_h)\cdot
\left(\e_{\u}\trans\B_h(o_{h},a_{h};\theta)\cdots\B_1(o_1,a_1;\theta)\b_0(\theta)\right),\\
	\P_\theta (o_{1:m}=\u_o \mid\ a_{1:m-1}=\u_a)  = \e_{\u}\trans\b_0(\theta),
\end{cases}
\end{equation}
where
$\e_\u$ is the standard basis vector in $\R^{|\U|}$ corresponding to $\u$,
\[
\left(\{\B_h(o,a;\theta)\}_{(o,a,h)\in\fO\times\fA\times[H-m]},\b_0(\theta)\right)
\] denote the operators corresponding to $\theta$, and $\P_\theta(o_{h+1:h+m}=\u_o\mid \tau_h,a_{h+1:h+m-1}=\u_a)$ denotes the probability of observing $\u_o$ at steps $h+1$ to $h+m$ conditioning on the first-$h$-steps history $\tau_h$ and the player following fixed action sequence $\u_a$ from step $h+1$ to $h+m-1$.

For a matrix $\A$, $p\ge 1$, let $\|\A\|_{p} = \sup_{x:\norm{x}_p\le 1} \norm{\A x}_p$.
We start with a lemma that bounds $\|\B_h(o,a)\|_{1}$ as a function of $\alpha$:
\begin{lemma}\label{lem:11bound-over}
Under Assumption~\ref{asp:over} it holds that
for any $(o,a,h,k)\in \fO\times \fA\times[H-m] \times \N$, 
$$\|\B_h(o,a)\|_{1},\|\B_h^k(o,a)\|_{1}  \le A^{m-1}\sqrt{S}/\alpha.$$
\end{lemma} 
We omint the proof of Lemma \ref{lem:11bound-over} here since it is basically the same as that of Lemma \ref{lem:11bound}.


\subsection{Step $1$: bound the regret by the error of operator estimates}

To begin with, we upper bound the cumulative regret by the summation of density estimation errors using Lemma \ref{lem:opt2dist}, 
\begin{equation}\label{eq:optimisim-over}
\sum_{t=1}^k V^{\pi^t}({\theta^t}) - V^{\pi^t}(\theta^\star)
\le  H \sum_{t=1}^k \sum_{\tau_H   } | \P^{\pi^t}_{{\theta^t}}(\tau_H) - \P^{\pi^t}_{\theta^\star}(\tau_H)|.
\end{equation}
As a result, to prove Theorem \ref{thm:over},  it suffices to upper bound the RHS of equation \eqref{eq:optimisim-over}.

Using the first relation in equation \eqref{eq:prob-op-over}, we can represent the probability of observing a trajectory by the product of operators, which gives 
\begin{align*}
	&  \sum_{t=1}^k \sum_{\tau_H   } \left| \P^{\pi^t}_{{\theta^t}}(\tau_H) - \P^{\pi^t}_{\theta^\star}(\tau_H)\right|\\
= &  \sum_{t=1}^k \sum_{\tau_H   } \bigg| \e_{(o_{H-m+1},\ldots,a_{H-1},o_H)}\trans \B^t_{H-m}(o_{H-m},a_{H-m})\cdots \B^t_1(o_1,a_1)\b^t_0 \\
&\qquad\qquad\qquad\qquad -\e_{(o_{H-m+1},\ldots,a_{H-1},o_H)}\trans \B_{H-m}(o_{H-m},a_{H-m})\cdots \B_1(o_1,a_1)\b_0\bigg| \times \pi^t(\tau_H)\\
\le & \sum_{t=1}^k \sum_{\tau_{H-m}   } \bigg\|  \B^t_{H-m}(o_{H-m},a_{H-m})\cdots \B^t_1(o_1,a_1)\b^t_0 \\
&\qquad\qquad\qquad\qquad - 
	\B_{H-m}(o_{H-m},a_{H-m})\cdots \B_1(o_1,a_1)\b_0\bigg\|_1 \times \pi^t(\tau_{H-m}).
\end{align*}
By Lemma \ref{lem:prod-triangle-over} (a general version of Lemma \ref{lem:prod-triangle-under}), we can  control the difference between the products of operators by the difference between each pair of operators.
\begin{align*}
\MoveEqLeft
\sum_{\tau_{H-m}  } \left\| \B^t_{H-m}(o_{  H-m},a_{  H-m})\cdots \B^t_1(o_1,a_1)\b^t_0 -  \B_{H-m}(o_{H-m},a_{H-m})\cdots \B_1(o_1,a_1)\b_0\right\|_1	\times \pi^t(\tau_{H-m})\\
&\le  \frac{A^{m-1}\sqrt{S}}{\alpha}  \left(\sum_{h=1}^{  H-m} \sum_{\tau_{h}   } 
\left\| \left(\B^t_h(o_h,a_h) - \B_h(o_h,a_h)\right)\b(\tau_{h-1})\right\|_1 \times \pi^t(\tau_{h}) + \left\|\b^t_0-\b_0\right\|_1\right),
\end{align*}
where $
\b(\tau_h) = 
\left(\prod_{h'=1}^h \B_{h'}(o_{h'},a_{h'}) \right)\b_0.
$

\subsection{Step $2$: derive constraints for the operator estimates from OMLE}

For simplicity of notation, we denote by $\mu^{k,h}$ the policy executed in the $h^{\rm th}$ inner loop  of  the $k^{\rm th}$ outer loop, i.e.,
$$
\mu^{k,h} = \pi^k_{1:h}\circ \left[\text{Uniform}(\fA)\right]_{h+1:H}
$$
By Proposition \ref{prop:mle-valid}, the definition of the confidence set and the choice of $\beta$, we have with probability at least $1-\delta$, 
\begin{equation*}
\text{ for all $k\in[K]$ } \quad 
\sum_{t=1}^{k-1} \sum_{h=0}^{H-m}  \left(\sum_{\tau_H}\left| \P^{\mu^{t,h}}_{\theta^k}(\tau_H) - \P^{\mu^{t,h}}_{\theta^\star}(\tau_H)\right|\right)^2 = \Ocal( \beta).
\end{equation*}

In what follows we assume  the above equation holds.
By the Cauchy-Schwarz inequality, 
\begin{equation*}
\sum_{t=1}^{k-1} \sum_{h=0}^{H-m}  \sum_{\tau_H}\left| \P^{\mu^{t,h}}_{\theta^k}(\tau_H) - \P^{\mu^{t,h}}_{\theta^\star}(\tau_H)\right| =\Ocal( \sqrt{\beta kH}).
\end{equation*}
Because marginalizing distributions cannot increase their distance,
and using the definition of $\mu^{t,h}$, the above relation implies that for any $(k,h)\in[K]\times[H-m]\times[H-m]$,
\begin{equation*}
\begin{cases}
&\frac{1}{A^{m-1}}\sum_{t=1}^{k-1}  \sum_{\tau_h,\u\in\U} \bigg| \P^{\pi^{t}}_{\theta^k}(\tau_{h}) \P_{\theta^k}(o_{h+1:h+m}=\u_o\mid \tau_h,a_{h+1:h+m-1}=\u_a)\\
&\qquad - \P^{\pi^{t}}_{\theta^\star}(\tau_{h}) \P_{\theta^\star}(o_{h+1:h+m}=\u_o\mid \tau_h,a_{h+1:h+m-1}=\u_a) \bigg| = \Ocal( \sqrt{\beta kH}),\\
&\frac{1}{A^{m}}\sum_{t=1}^{k-1}  \sum_{\tau_h,\u\in\U} \bigg| \P^{\pi^{t}}_{\theta^k}(\tau_{h-1}) \P_{\theta^k}(o_{h+1:h+m}=\u_o\mid \tau_h,a_{h+1:h+m-1}=\u_a)\\
&\qquad - \P^{\pi^{t}}_{\theta^\star}(\tau_{h-1}) \P_{\theta^\star}(o_{h+1:h+m}=\u_o\mid \tau_h,a_{h+1:h+m-1}=\u_a) \bigg| = \Ocal( \sqrt{\beta kH}),\\
&\frac{1}{A^{m-1}}\sum_{t=1}^{k-1}  \sum_{\u\in\U} \bigg|  \P_{\theta^k} (o_{1:m}=\u_o \mid\ a_{1:m-1}=\u_a)  \\
&\qquad\qquad-  \P_{\theta^\star} (o_{1:m}=\u_o \mid\ a_{1:m-1}=\u_a)\bigg|_1 = \Ocal( \sqrt{\beta kH}).
\end{cases}
\end{equation*}
By using the two relations in Equation \eqref{eq:prob-op-over}, we can replace the probability with the product of operators and obtain that for any $(k,h)\in[K]\times[H-m]$
\begin{equation}\label{eq:mle-1-over}
\begin{cases}
\sum_{t=1}^{k-1}  \sum_{\tau_h } \pi^t(\tau_h) \times \left \| \b^k(\tau_h) - \b(\tau_h)\right\|_1	= \Ocal( A^{m-1}\sqrt{\beta kH}), \\
\sum_{t=1}^{k-1}  \sum_{\tau_h } \pi^t(\tau_{h-1}) \times \left \| \b^k(\tau_h) - \b(\tau_h)\right\|_1	= \Ocal( A^{m}\sqrt{\beta kH}), \\
\left \| \b_0^k - \b_0 \right \|_1 = \Ocal( A^{m-1}\sqrt{\beta H/k}),
\end{cases}
\end{equation}
where $
\b(\tau_h) = 
\left(\prod_{h'=1}^h \B_{h'}(o_{h'},a_{h'}) \right)\b_0
$ and $
\b^k(\tau_h) = 
\left(\prod_{h'=1}^h \B^k_{h'}(o_{h'},a_{h'}) \right)\b_0^k.
$

 Now, we are ready to derive the  final guarantee for each individual operator estimate $\B^k_h(o,a)$. For all $(k,h)\in[K]\times[H-m]$,
\begin{equation}\label{eq:mle-2-over}
\begin{aligned}
	\MoveEqLeft\sum_{t=1}^{k-1}  \sum_{ \tau_{h}} \pi^t(\tau_{h-1}) \times 
\left\| \left(\B^k_h(o_h,a_h) - \B_h(o_h,a_h)\right)\b(\tau_{h-1})\right\|_1\\
& \le   \sum_{t=1}^{k-1}   \sum_{ \tau_h} \pi^t(\tau_{h-1}) \times 
	\left\| \B^k_h(o_h,a_h) \left(\b^k(\tau_{h-1})-\b(\tau_{h-1})\right)\right\|_1  \\
	&\quad + 
	 \sum_{t=1}^{k-1}  \sum_{ \tau_h} \pi^t(\tau_{h-1}) \times
	\left\| \B^k_h(o_h,a_h)\b^k(\tau_{h-1}) - \B_h(o_h,a_h)\b(\tau_{h-1})\right\|_1.
\end{aligned}
\end{equation}
For the first term in the RHS of Equation \eqref{eq:mle-2-over}: 
\begin{align*}
\MoveEqLeft	
\sum_{t=1}^{k-1}   \sum_{ \tau_h} \pi^t(\tau_{h-1}) \times 
	\left\| \B^k_h(o_h,a_h) \left(\b^k(\tau_{h-1})-\b(\tau_{h-1})\right)\right\|_1 
	\\
	& = A \sum_{t=1}^{k-1}   \sum_{ \tau_{h}} \pi^{t,h-1}(\tau_{h}) \times 
	\left\| \B^k_h(o_h,a_h) \left(\b^k(\tau_{h-1})-\b(\tau_{h-1})\right)\right\|_1 \\
		\text{by Lemma \ref{lem:OP3}} & \le  \frac{A^m\sqrt{S}}{\alpha} \times \sum_{t=1}^{k-1}   \sum_{ \tau_{h-1}} \pi^{t,h-1}(\tau_{h-1}) \times 
	\left\| \b^k(\tau_{h-1})-\b(\tau_{h-1})\right\|_1     \\
	&= \frac{A^m\sqrt{S}}{\alpha} \times \sum_{t=1}^{k-1}   \sum_{ \tau_{h-1}} \pi^{t}(\tau_{h-1}) \times 
	\left\| \b^k(\tau_{h-1})-\b(\tau_{h-1})\right\|_1 \\
	&= \Ocal\left( \frac{A^{2m-1}\sqrt{S}}{\alpha}\sqrt{\beta kH}\right),
\end{align*}
where the last equality uses the first  relation in Equation \eqref{eq:mle-1-over}.

For the second term in the RHS of Equation \eqref{eq:mle-2-over}: by the second  relation in Equation \eqref{eq:mle-1-over}. 
\begin{align*}
\sum_{t=1}^{k-1}  \sum_{ \tau_h} \pi^t(\tau_{h-1}) \times
\left\| \B^k_h(o_h,a_h)\b^k(\tau_{h-1}) - \B_h(o_h,a_h)\b(\tau_{h-1})\right\|_1
= \Ocal( A^{m}\sqrt{\beta kH}).
\end{align*}
Therefore, we conclude that 
 for all $(k,h)\in[K]\times[H-m]$,
$$
\sum_{t=1}^{k-1}  \sum_{ \tau_{h}} \pi^t(\tau_{h-1}) \times 
\left\| \left(\B^k_h(o_h,a_h) - \B_h(o_h,a_h)\right)\b(\tau_{h-1})\right\|_1= \Ocal\left( \frac{A^{2m-1}\sqrt{S}}{\alpha}\sqrt{\beta kH}\right).
$$

\subsection{Step $3$: bridge Step 1 and 2 via $\ell_1$-norm eluder dimension}

By Step $1$, to control the density estimation error, it suffices to upper bound the following quantity for all $k\in[K]$:
$$
\textbf{Target: }\sum_{t=1}^k \left(\sum_{h=1}^{  H-m} \sum_{\tau_{h}   } 
\left\| \left(\B^t_h(o_h,a_h) - \B_h(o_h,a_h)\right)\b(\tau_{h-1})\right\|_1 \times \pi^t(\tau_{h-1}) + \left\|\b^t_0-\b_0\right\|_1\right).
$$
By Step $2$,   we have with probability at least $1-\delta$: for all  $(k,h)\in[K]\times[H-m]$
\begin{equation*}
\textbf{Condition: }
\begin{cases}
	&\left \| \b_0^k - \b_0 \right \|_1 = \Ocal( A^{m-1}\sqrt{\beta H/k}), \\
	&\sum_{t=1}^{k-1}  \sum_{ \tau_h} 
	\left\| \left(\B^k_h(o_h,a_h) - \B_h(o_h,a_h)\right)\b(\tau_{h-1})\right\|_1 \times \pi^t(\tau_{h-1})=  \Ocal\left( \frac{A^{2m-1}\sqrt{S}}{\alpha}\sqrt{\beta kH}\right).
\end{cases}
\end{equation*}
Below, we will carry on the proof conditioning on the event that the above relations hold. 

The summation of the second term in the target can be upper bounded by using the first condition:
$$\cO\left(\sum_{t=1}^k A^{m-1}\sqrt{\beta H/t} \right) = \cO\left(A^{m-1}\sqrt{\beta k H} \right).$$
It remains to control the summation of the first term.

Let's consider an \emph{arbitrary fixed}  $h\in[  H-m]$. 
Denote by $\X_{l}$ the $l^{\rm th}$ row of matrix $\X$. The second  condition is equivalent to 
\begin{equation}\label{eq:step3-1-over}
\begin{aligned}
	\MoveEqLeft\sum_{t=1}^{k-1}  \sum_{ l=1 }^{O^{m}A^{m-1}} \sum_{o,a} \sum_{\tau_{h-1}}
	\left|  \left[\left(\B^k_h(o,a) - \B_h(o,a)\right)\M_h\right]_{l} \M_{h}^\dagger\b(\tau_{h-1})\right| \times \pi^t(\tau_{h-1}) \\
	&= \Ocal\left( \frac{A^{2m-1}\sqrt{S}}{\alpha}\sqrt{\beta kH}\right).
\end{aligned}
\end{equation}

To further simplify the notations, for $(t,l,o,a)\in [K]\times[O^m A^{m-1}]\times\fO\times\fA$, let $w_{t,l,o,a} :=\left[\left(\B^t_h(o,a) - \B_h(o,a)\right)\M_h\right]_{l}$  and 
denote
 the $n:=(OA)^{h-1}$ not necessarily distinct elements in the 
sequence
$\{ \M_h^\dagger\b(\tau_{h-1})\times \pi^t(\tau_{h}): \ \tau_h=(o_1,a_1,\dots,o_{h-1},a_{h-1},o,a)\in (\fO\times\fA)^{h-1}\times(o,a) \}$
by $x_{t,1},\dots,x_{t,n}$. 
Using the newly defined notations, Equation \eqref{eq:step3-1-over} is equivalent to
\begin{equation}\label{eq:step3-2-over}
	\sum_{t=1}^{k-1} \sum_{l=1}^{O^m A^{m-1}}\sum_{o,a}\sum_{i=1}^n
| w_{k,l,o,a}\trans x_{t,i}| = \Ocal\left( \frac{A^{2m-1}\sqrt{S}}{\alpha}\sqrt{\beta kH}\right).
\end{equation}

Now, we invoke  Proposition \ref{prop:linear-regret} with Equation \eqref{eq:step3-2-over} and the following normalization conditions: for all $t\in \N$
\begin{itemize}
	\item The vectors  $\{x_{t,i}\}_{i=1}^n$ satisfy $\sum_{i=1}^n \|x_{t,i}\|_1 \le 1$, which follows  from the same argument as in Appendix \ref{subsec:step3}),
	\item For the vectors  $\{w_{t,l,o,a}\}_{l,o,a}$, we have 
	\begin{align*}\sum_{l=1}^{O^m A^{m-1}}\sum_{o,a}
	\| w_{t,l,o,a}\|_1 
	=  \sum_{(o,a)\in\fO\times\fA} \sum_{s=1}^S \| (\B_h(o,a)-\B_h^t(o,a)) \M_{h} \e_s\|_1
	\le   \frac{2S^{1.5}A^{2m-1}}{\alpha},
	\end{align*}
	where the inequality uses $ \|\M_{h} \e_s\|_1=A^{m-1}$ and  Lemma \ref{lem:OP3} with $\pi=\text{Uniform}(\fA)$.
\end{itemize}
As a result, we obtain 
\begin{equation}
	\sum_{t=1}^{k} \sum_{l=1}^{O^m A^{m-1}} \sum_{i=1}^n
| w_{t,l}\trans x_{t,i}| = \tilde{\Ocal}\left( \frac{S^{1.5}H^2A^{2m-1}}{\alpha}\sqrt{\beta kH}\right), 
\ \text{ for all } k\in[K]. 
\end{equation}
Translating the above guarantee back using the operator language  gives that for all $k\in[K]$,
\begin{align*}
\MoveEqLeft \sum_{t=1}^k\left(\sum_{h=1}^{  H-1} \sum_{\tau_{h}   } 
\left\| \left(\B^t_h(o_h,a_h) - \B_h(o_h,a_h)\right)\b(\tau_{h-1})\right\|_1 \times \pi^t(\tau_{h-1}) + \left\|\b^t_0-\b_0\right\|_1\right)\\
&= \tilde{\Ocal}\left( \frac{S^{1.5}H^3A^{2m-1}}{\alpha}\sqrt{\beta kH}\right).
\end{align*}
Combining the above equality with Step $1$ gives:
$$   \sum_{t=1}^k \left( V^{\star}-V^{\pi^t} \right) \le \tilde{\Ocal}\left( \frac{S^2 A^{3m-2} }{\alpha^2}\sqrt{k(S^2A+SO)}\times\poly(H)\right) \qquad \text{ for all }k\in[K].$$

\subsection{Auxiliary lemmas}
\label{sec:auxlemmas}
We prove the following auxiliary lemma for general POMDPs. 
The definitions of the notations can be found in the beginning of Appendix \ref{app:proof-over}.

\begin{lemma}\label{lem:prod-triangle-over}
Suppose Assumption \ref{asp:over} holds.
For any $k\in\N$, $h\in[H-m]$ and policy $\pi$
\begin{align*}
&\sum_{\tau_{h}  } \left\| \B^k_h(o_{  h},a_{  h})\cdots \B^k_1(o_1,a_1)\b^k_0 -  \B_h(o_h,a_H)\cdots \B_1(o_1,a_1)\b_0\right\|_1	\times \pi(\tau_{h})\\
\le & \frac{A^{m-1}\sqrt{S}}{\alpha}  \left(\sum_{j=1}^{h} \sum_{\tau_{j}  } 
\left\| (\B^k_j(o_j,a_j) - \B_j(o_j,a_j))\b(\tau_{j-1})\right\|_1 \times \pi(\tau_{j}) + \|\b^k_0-\b_0\|_1\right),
\end{align*}
where $
\b(\tau_h) = 
\left(\prod_{h'=1}^h \B_{h'}(o_{h'},a_{h'}) \right)\b_0$.
\end{lemma}
\begin{proof}[Proof of Lemma \ref{lem:prod-triangle-over}]
To reduce clutter, we  abbreviate  
$\B^k_h(o_{ h},a_{ h})$ as $\B^k_h$    
and denote $\B^k_{h:j}:=\B^k_{h}\times\cdots\times \B^k_{j}$. 
By triangle inequality, 
\begin{align*}
&\sum_{\tau_{h}  } \left\| \B^k_{ h}\cdots \B^k_1\b^k_0 -  \B_{h}\cdots \B_1\b_0\right\|_1	\times \pi(\tau_{h})\\
\le &   \sum_{j=1}^{h} \sum_{\tau_{h}  }
\|\B_{h:j+1}^k (\B^k_j - \B_j)\b(\tau_{j-1})\|_1 \times \pi(\tau_{h}) + 
\sum_{\tau_{h}  } \|\B_{h:1}^k(\b^k_0-\b_0)\|_1\times \pi(\tau_{h}).
\end{align*}
	So it suffices prove the following Lemma:
	\begin{lemma}\label{lem:OP3}
	For any  index $ j < h \le H-m$, trajectory $\tau_j\in (\fO\times\fA)^j$, vector  $x\in\R^{O^m A^{m-1}}$, policy $\pi$, and operator $\tilde{\B}\in\{\B,\B^k\}$,  we have 
	$$
	\sum_{\tau_{h:j+1}\in(\fO\times\fA)^{h-j} }
		\|\tilde{\B}_{h:j+1} x \|_1 \times \pi(\tau_{h:j+1}\mid \tau_j) 
	\le \frac{A^{m-1}\sqrt{S}}{\alpha}  \|x\|_1.
	$$
	\end{lemma}
\begin{proof}[Proof of Lemma \ref{lem:OP3}]
We only prove the lemma for  $\tilde\B = \B^k$ since the other case follows exactly the same arguments.
By the definition of $\B^k_{j+1}$, we have that the row space of $\B_{h:j+1}^k$ belongs to the column space of $\M_{j+1}^k$, which implies 
$$
\B_{h:j+1}^k x = \B_{h:j+1}^k \M^k_{j+1} (\M^k_{j+1})^\dagger x.
$$
Moreover, for any standard basis $\e_i \in\R^{S}$
\begin{align*}
&	\sum_{\tau_{h:j+1}  } \| \B_{h:j+1}^k \M^k_j \e_i \|_1 \times  \pi(\tau_{h:j+1}\mid \tau_j)  \\
=  & \sum_{\a\in\fA^{m-1}}\sum_{\o\in\fO^m}\sum_{\tau_{h:j+1}  } \P^{\pi}_{\theta^k}(\tau_{h:j+1}\mid s_{j+1}=i,\ \tau_j)
\cdot \P_{\theta^k}(o_{h+1:h+m}=\o\mid \tau_h,a_{h+1:h+m-1}=\a)\\
=  & A^{m-1}.
\end{align*}
Combining all relations together, we obtain 
$$
	\sum_{\tau_{h:j+1} }
		\|\B_{h:j+1}^k x \|_1 \times \pi^k(\tau_{h:j+1}\mid \tau_j) 
	\le A^{m-1} \|  (\M^k_{j+1})^\dagger x\|_1 \le 
	\frac{A^{m-1}\sqrt{S}}{\alpha}   \|x\|_1,
	$$
where the final inequality follows from 
Assumption \ref{asp:over}.
\end{proof}
\renewcommand{\qedsymbol}{}
\end{proof}

\section{Proofs for Lower Bounds}
\label{app:lower}
In this section, we prove  the  two lower bounds presented in Section \ref{sec:main}.

\subsection{Proof of Theorem \ref{thm:lowerbound1}}

\begin{proof}[Proof of Theorem \ref{thm:lowerbound1}]
We construct the hard instance based on  combinatorial lock, which we define formally as follows.

\begin{enumerate}
	\item STATE: We have $2$ states, with $1$ ``good state" $s_g$ and $1$ ``bad state" $s_b$. 
	
	\item OBSERVATION: There are three observations $o_g, o_b$ and $o_{\rm dummy}$. The emission matrices in the first $H-1$ steps are 
	$$
\begin{pmatrix}
	\alpha & 0\\
	 0 & \alpha\\
	 1-\alpha & 1-\alpha\\
\end{pmatrix},
$$
while the emission matrix at step $H$ is 
$$
\begin{pmatrix}
	1 & 0\\
	 0 & 1\\
	 0 & 0\\
\end{pmatrix}.$$
	In other words, at each step $h\in[H-1]$,  with probability $\alpha$ we observe the current latent state,  and with probability $1-\alpha$ we receive a dummy observation.  That is, with probability $\alpha$ we find out what state we are currently at, and with probability $1 - \alpha$, we learn nothing. But at step $H$, we always directly observe the current latent state. 
	By standard linear algebra, one can verify $\min_h \sigma_S(\O_h) \ge\alpha$.

	\item ACTION AND TRANSITION: There are $A$ different actions and the initial state at step $1$ is fixed as $s_{g}$. We now define the transitions as follows. For every $h\in [H-1]$, we designate one of these actions to be ``good", the others``bad". If one is currently in the``good" state and takes the ``good" action, then the agent will transition to the ``good" state. Otherwise, one will always transition to the ``bad" state. For each $h  \in [H-1]$, the good action is chosen \emph{uniformly at random} from $\fA$. The current episode terminates immediately after $o_H$ is received.
\item REWARD: Viewing every observation at the first $H-1$ steps yields a reward of $0$.  At step $H$, viewing $o_g$ yields reward $1$ while viewing $o_b$ yields reward $0$. 
 As a result, the agent receives reward $1$ only if the state $s_{g}$ is reached at step $H$, i.e., if the good action is taken at every step in the POMDP. 
\end{enumerate}

\noindent \textbf{Showing large number of samples are necessary:} Suppose that we attempt to learn this POMDP with an algorithm $\mathcal{A}$, where we are allowed to interact with the POMDP for  $ K \le \lfloor\frac{1}{2\alpha H}\rfloor$ episodes. Now, consider the probability that in these $K$ episodes, \emph{both $s_g$ and $s_b$  only emit dummy observations}  in the first $H-1$ steps and the agent always gets reward $0$. We can write this as 
\begin{align*}
&\P(\text{only emit dummy observations in the first $H-1$ steps and get reward 0 in all episodes})\\
 = & \P(\text{get reward 0 in all episodes} \mid \text{only emit dummy observations in the first $H-1$ steps }) \\
 & \times  \P(\text{only emit dummy observations in the first $H-1$ steps }).
\end{align*} 

 The probability that both $s_g$ and $s_b$  only emit dummy observations  in the first $H-1$  steps of all $K$ episodes  is upper bounded by $(1-\alpha)^{1/\alpha}$ because $s_g$ and $s_b$ together can emit at most $2\times \lfloor\frac{1}{2\alpha H}\rfloor \times (H-1)$ observations in the first $H-1$  steps of $K$ episodes.

Now, conditioned on the event that both $s_g$ and $s_b$  only emit dummy observations  in the first $H-1$  steps of all $K$ episodes, in each episode, the best strategy we can use is to randomly guess the sequence of optimal actions. In particular, the probability that $\mathcal{A}$ fails to guess the optimal sequence correctly, given that we have $K$ guesses, is $\binom{A^{H-1}-1}{K}/\binom{A^{H-1}}{K} = \frac{A^{H-1}-K}{A^{H-1}}.$ Then, for $K \le  A^{H-1}/10$, this is 
at least $9/10$. Thus, with probability at least $0.9\times (1-\alpha)^{1/\alpha} \ge 1/6$, the agent learns nothing except that the action sequences it chose are incorrect, and because it only sampled at most one tenth of the total possible action sequences by assumption, the best policy it can output is just to randomly guess from the remaining action  sequences, which is worse than $\left(1/2\right)$-optimal, which completes the proof.
\end{proof}

\begin{lemma}
For $0<\alpha \leq \frac{1}{2},$ $(1-\alpha)^{1/\alpha} \ge \frac{1}{5}.$
\end{lemma}

\begin{proof}
It suffices to prove $$1 - \alpha \ge \frac{1}{5^\alpha} \Leftrightarrow \alpha + ( \frac{1}{5^\alpha}-1) \leq 0.$$ We note that the LHS is an increasing function and that for $\alpha = \frac{1}{2},$ the LHS is $$\frac{1}{\sqrt{5}} - \frac{1}{2} \leq 0,$$ which suffices to show the desired inequality for all $\alpha \leq \frac{1}{2}.$
\end{proof}

\subsection{Lower bound for  the exponential dependency on $m$}

\begin{proof}[Proof of Theorem \ref{thm:lowerbound2}]
We construct the hard instance based on  combinatorial lock, which we define formally as follows.  We choose the episode length $H$ equal to $m$.

\begin{enumerate}
	\item STATE: We have $2$ states, with $1$ ``good state" $s_g$ and $1$ ``bad state" $s_b$. 
	
	\item OBSERVATION AND REWARD: We have three observations.

\begin{enumerate}
    \item Dummy Observation: at the first $m-1$ steps, we view this observation with probability 1 regardless of our current latent state, which gives reward $0$.
    \item Reward 1 Observation: at step $m$ and state $s_g$ we view this observation with probability 1, which gives reward $1$.
    \item Reward 0 Observation: at step $m$ and state $s_b$ we view this observation with probability 1, which gives reward $0$.
\end{enumerate}

\item ACTIONS AND TRANSITIONS: \textit{We define the actions and transitions identically as in the proof of Theorem \ref{thm:lowerbound1}}. 

\end{enumerate}

It is easy to verify   the $m$-step emission-action matrix at step $h=1$ has rank $2$ and minimum singular value no smaller than  $1$ (by using Proposition \ref{prop:property_multistep_weak_revealing} with the optimal action sequence). Therefore, it satisfies Assumption \ref{asp:over} with $\alpha\ge1$. (Since $H=m$, Assumption \ref{asp:over} only requires $\sigma_S(\M_1) \ge 1$.)
Moreover, we claim that the agent must take $\Omega(A^{m-1})$ samples in order to learn this POMDP. This is because in the first $m-1$ steps, the reward is always 0 and the observation is always dummy, so that no knowledge is learned. The only useful information is the reward feedback in the final step.
Therefore, this is equivalent to a multi-arm bandit problem with $A^{m-1}$ arms. 
\end{proof}

\section{Proofs for Weakly Revealing Conditions}
\label{app:weakly}

\subsection{Proof of Proposition \ref{prop:property_singlestep_weak_revealing}}

\begin{proof}
By the definition of minimum singular value, $\sigma_{S}(\O_h)\neq0$ implies that for any $\mu \neq \mu'$, $\O_h(\mu-\mu')\neq \mathbf{0}$.
On the other hand, if $\sigma_{S}(\O_h)= 0$, then there exists $z\neq \mathbf{0}$ such that $\O_h z = \mathbf{0}$. Let $z^{+} = \max\{z,0\}$ and $z^{-} = -\min\{z,0\}$, where the max and min are taken entry-wisely. 
By definition, $\O_h z^+ = \O_h z^-$. Moreover, $\|z^+\|_1 = \|\O_h z^+\|_1 = \|\O_h z^-\|_1= \|z^{-}\|_1$. As a result, we have 
$\O_h \frac{z^+}{\|z^+\|_1} = \O_h \frac{z^-}{\|z^+\|_1}$ where $\frac{z^+}{\|z^+\|_1}$ and $\frac{z^-}{\|z^-\|_1}$ are two disjoint state distributions that induce the same  distribution over observations.
\end{proof}

\subsection{Proof of Proposition \ref{prop:property_multistep_weak_revealing}}

\begin{proof}
    We prove Proposition \ref{prop:property_multistep_weak_revealing} by showing that for any fixed $h$, $\sigma_S(\M_h) \ge  \max_{a\in \fA^{m-1}} \sigma_S(\M_{h,\a})$. By the Courant-Fischer-Weyl min-max principle,
    for matrix $U\in \R^{n\times S}$ with $n\ge S$,
    \[
    \sigma_S^2(U) = \min_{y\in \R^S:\|y\|_2 = 1} \| U y \|_2^2 
    =
     \min_{y\in \R^S:\|y\|_2 = 1} \sum_{i=1}^n ( U_{i:} y )^2\,,
    \]
    where $U_{i:}$ stands for the $i$th row of $U$.
    Fix $h\in [H]$ and by abusing notation let $\M = \M_h$.
    Then,
    \begin{align*}
    \sigma_S^2(\M)
    &= \min_{y\in \R^S:\|y\|_2 = 1} \sum_{\a} \sum_{\o} ( \M_{(\a,\o):} y )^2
    \ge
    \sum_{\a} \min_{y\in \R^S:\|y\|_2 = 1}  \sum_{\o} ( \M_{(\a,\o):} y )^2\\
    & \ge
    \max_{\a} \sigma_S^2(\M_{h,\a})\,.
    \end{align*}
    \end{proof}

\subsection{Relation between the $\gamma$-observability \citep{golowich2022planning} and the  weakly revealing conditions}

\begin{lemma}\label{lem:basic}
Suppose $\phi,\psi \in\R^d$ satisfy 
$\sum_i \phi_i = 0$ and $(\max_i \psi_i)\times(\min_i \psi_i) \ge 0$, then  
$$\sum_i |\phi_i+\psi_i| \ge \frac{1}{2}\max\{\|\phi\|_1,\|\psi\|_1\}.$$
\end{lemma}
\begin{proof}
WLOG, assume $\min_i \psi_i\ge 0$.
By triangle inequality,
    $\sum_{i} | \phi_i + \psi_i| \ge |\sum_{i} \phi_i + \psi_i| = \|\psi\|_1$. 
Moreover, $\sum_i |\phi_i+\psi_i| \ge \sum_{i:\phi_i\ge 0} (\phi_i+\psi_i )\ge \frac12 \|\phi\|_1$.
\end{proof}
\begin{lemma}\label{lem:equiv}
Suppose $\gamma$ and $\alpha$ are the largest real numbers that satisfy
\begin{enumerate}
    \item For any $v_1,v_2\in\Delta_S$, $\|\O_h(v_1-v_2)\|_1 \ge \gamma \| v_1 - v_2\|_1$,
    \item $\sigma_{S}(\O_h) \ge  \alpha$.
\end{enumerate}
Then $\frac{\alpha}{\sqrt{S}} \le \gamma \le 4\sqrt{O}\alpha$.
\end{lemma}
\begin{proof}
Suppose $\sigma_{S}(\O_h) \ge  \alpha$. We have
$$
\|\O_h(v_1-v_2)\|_1 \ge 
\|\O_h(v_1-v_2)\|_2 \ge
\alpha  \|v_1-v_2\|_2
\ge \frac{\alpha}{\sqrt{S}}  \|v_1-v_2\|_1,
$$
which implies $\gamma \ge \frac{\alpha}{\sqrt{S}}$.

Suppose for any $v_1,v_2\in\Delta_S$, $\|\O_h(v_1-v_2)\|_1 \ge \gamma \| v_1 - v_2\|_1$. Consider an arbitrary $z\in \R^O$ and decompose $z$ as $z=z^+ - z^- + \bar{z}$, where $\bar{z} = \frac{\sum_{i}z_i}{O}\times \mathbf{1}$,  $z^+  = \max\{z-\bar{z},0\}$ and $z^-= \max\{\bar{z}-z,0\}$. 
Invoking Lemma \ref{lem:basic} with 
$\phi = \O_h(z^+ - z^-)$ and $\psi = \O_h\bar{z}$, we obtain
$$
\| \O_h z \|_2 \ge \frac{1}{2\sqrt{O}} \max\{ \| \O_h(z^+ - z^-)\|_1,  \|\O_h\bar{z}\|_1\}.
$$
Note that $\frac{z^+}{\|z^+\|_1}, \frac{z^0}{\|z^-\|_1}\in\Delta_S$  and $\|z^+\|_1 = \|z^-\|_1$, so we have $\|\O_h(z^+ - z^-)\|_1 \ge {\gamma}\|z^+ - z^-\|_2$. Besides, by the definition of $\bar{z}$ and $\O_h$, $\|\O_h\bar{z}\|_1= \|\bar{z}\|_1$, which implies 
$\|\O_h\bar{z}\|_1\ge \|\bar{z}\|_2$. 
As a result, we conclude that 
$$
\| \O_h z \|_2  \ge \frac{\gamma}{2\sqrt{O}}\max\{ \|z^+ - z^-\|_2,  \|\bar{z}\|_2\}  \ge  \frac{\gamma}{4\sqrt{O}} \| z \|_2.
$$
\end{proof}

\end{document}